\documentclass[letterpaper]{article} 
\usepackage{aaai24}  
\usepackage{times}  
\usepackage{helvet}  
\usepackage{courier}  
\usepackage[hyphens]{url}  
\usepackage{graphicx} 
\usepackage{natbib}  
\usepackage{caption} 
\usepackage{algorithm}
\usepackage{algorithmic}

\usepackage{newfloat}
\usepackage{listings}
\DeclareCaptionStyle{ruled}{labelfont=normalfont,labelsep=colon,strut=off} 
\floatstyle{ruled}
\newfloat{listing}{tb}{lst}{}
\floatname{listing}{Listing}
\usepackage{amssymb}
\usepackage{amsmath}
\usepackage{dsfont}
\usepackage{mathtools}
\usepackage{enumerate}
\usepackage{subfigure}
\usepackage{multirow}
\usepackage{bm}
\usepackage{colortbl}
\usepackage{xcolor}
\usepackage{comment}
\definecolor{Gray}{gray}{0.9} 
\usepackage{amsthm}
\newtheorem{theorem}{Theorem}
\newtheorem*{theorem*}{Theorem}

\newtheorem*{lemma*}{Lemma}
\newtheorem{definition}{Definition}
\newtheorem{proposition}{Proposition}
\newtheorem*{proposition*}{Proposition}
\newtheoremstyle{dotless}{}{}{\itshape}{}{\bfseries}{}{ }{}
\theoremstyle{dotless}
\newtheorem*{statement}{Statement:}

\title{Improve Robustness of Reinforcement Learning against Observation Perturbations via $l_\infty$ Lipschitz Policy Networks}

\title{Improve Robustness of Reinforcement Learning against Observation Perturbations via $l_\infty$ Lipschitz Policy Networks}
\author {
    Buqing Nie,
    Jingtian Ji,
    Yangqing Fu,
    Yue Gao\thanks{Corresponding author.}
}
\affiliations {
    MoE Key Lab of Artificial Intelligence and AI Institute,
    Shanghai Jiao Tong University\\
    \{niebuqing, jijingtian, frank79110, yuegao\}@sjtu.edu.cn
}

\usepackage{bibentry}

\begin{document}

\maketitle

\begin{abstract}
Deep Reinforcement Learning (DRL) has achieved remarkable advances in sequential decision tasks.
However, recent works have revealed that DRL agents are susceptible to slight perturbations in observations.
This vulnerability raises concerns regarding the effectiveness and robustness of deploying such agents in real-world applications. 
In this work, we propose a novel robust reinforcement learning method called \emph{SortRL}, which improves the robustness of DRL policies against observation perturbations from the perspective of the network architecture.
We employ a novel architecture for the policy network that incorporates global $l_\infty$ Lipschitz continuity and provide a convenient method to enhance policy robustness based on the output margin.
Besides, a training framework is designed for \emph{SortRL}, which solves given tasks 
while maintaining robustness against $l_\infty$ bounded perturbations on the observations.
Several experiments are conducted to evaluate the effectiveness of our method, including classic control tasks and video games.
The results demonstrate that \emph{SortRL} achieves state-of-the-art robustness performance against different perturbation strength.
\end{abstract}

\section{Introduction}
Recently, Deep Reinforcement Learning (DRL) has achieved breakthrough success in various application scenarios, including video games~\cite{mnih2015human}, recommender systems~\cite{afsar2022reinforcement}, and robotics control~\cite{lee2020learning}.
These achievements typically rely on the Deep Neural Networks (DNNs) as function approximators for their strong expressive power,  which enables the end-to-end learning of policies in complex environments with high-dimension state spaces, such as images observations~\cite{hornik1989multilayer, mnih2015human,DBLP:conf/iclr/KaiserBMOCCEFKL20}.

However, DNNs typically lack robustness due to their highly non-linear and black-box nature, resulting in unreasonable and unpredictable outputs when inputs are perturbed slightly~\cite{madry2018towards,yuan2019adversarial}.
Similarly, recent works have shown that typical DNN-based policies are also vulnerable to imperceptible perturbations on observations, also known as ``state adversaries'', which are prevalent in application scenarios such as sensor noise~\cite{zang2019impact} and adversarial attacks~\cite{DBLP:conf/iclr/HuangPGDA17}.
These slight perturbations can deceive typical DRL policies easily, leading to irrational and unpredictable decisions by the agent~\cite{fischer2019online,zhang2020robust,oikarinen2021robust,DBLP:conf/iclr/ZhangCBH21,sun2022who}.
This may affect the policy effectiveness and user experience, even causing safety issues, especially in safety-critical applications such as autonomous driving and robot manipulation tasks~\cite{zhao2022cadre}.
The lack of robustness to observation perturbations renders applications of DRL unreliable and risky, thereby limiting potential applications in real-world scenarios.

In the recent decade, plenty of works have been proposed to certify and enhance the robustness of DRL policies against perturbations on observations.
Some researchers propose various robust policy regularizers to enforce policy smoothness, i.e. the policy output similar actions given similar observations~\cite{zhang2020robust,shen2020deep,oikarinen2021robust}.
For example, Shen et al.~\cite{shen2020deep} propose a smoothness-inducing regularizer inspired by Lipschitz continuity to encourage the policy function to become smooth, which improves sample efficiency and policy robustness in continuous control tasks.
Despite the excellent performance achieved, the incorporation of a smoothness regularizer may hinder the expressive power of the policy network, resulting in a partial compromise of optimality and performance, especially in tasks with strong perturbation strength~\cite{wu2022robust}.

Another approach to enhancing the policy robustness is based on attacking and adversarial samples~\cite{mandlekar2017adversarially,pattanaik2018robust,DBLP:conf/iclr/ZhangCBH21}.
For instance, Pattanaik et al.~\cite{pattanaik2018robust} improve policy robustness utilizing adversarial observations found by gradient-based attackers.
Recently, Zhang et al.~\cite{DBLP:conf/iclr/ZhangCBH21} propose Alternating Training with Learned Adversaries (ATLA), which trains an RL adversary online with the agent policy alternately.
ALTA significantly improves the policy robustness in continuous control tasks.
Despite the excellent robustness, these methods require training extra attackers or finding adversaries for the observations, which incurs additional computational and sampling costs, thereby limiting their practical applications.

In this work, we propose a novel method called \emph{SortRL} to improve the robustness of DRL policies against observation perturbations from the perspective of the network architecture.
We introduce a new policy network architecture based on an $l_\infty$ Lipschitz Neural Network called \emph{SortNet}.
Besides, we introduce a straightforward and efficient method to estimate the lower bound of policy robustness utilizing the output margin.
Additionally, we design a training framework for \emph{SortRL} based on Policy Distillation~\cite{DBLP:journals/corr/RusuCGDKPMKH15}, which enables the agent to solve the given tasks successfully while addressing robustness requirements against observation perturbations.
Several experiments on classic control tasks and video games are conducted to evaluate the performance of \emph{SortRL}, which demonstrates the state-of-the-art performance of our method.

Our main contributions are listed as follows:
\begin{itemize}
    \item We propose a novel robust reinforcement learning method called \emph{SortRL}, which enhances the policy robustness against observation perturbations.
    To our knowledge, this is the first work to address this issue from the perspective of network architecture.
    \item 
    We employ a novel policy design base on an $l_\infty$ Lipschitz Neural Network.
    A convenient method is provided to evaluate and improve policy robustness based on the output margin.
    \item We design a training framework for \emph{SortRL} to make a trade-off between optimality and robustness, which enables the agent to solve given tasks while addressing robustness requirements.
    \item Experiments on classic control tasks and video games are conducted, which demonstrate that \emph{SortRL} achieves state-of-the-art robustness against different perturbation strength, especially in tasks with strong perturbations.
    
\end{itemize}

\section{Related Work}
\subsection{Robust Reinforcement Learning}
Robust Reinforcement Learning aims to improve the policy robustness against perturbations in the Markov Decision Process (MDP).
Thus, there exist various interpretations of robustness in the RL context, including the robustness against action perturbations~\cite{tessler2019action}, dynamics uncertainty~\cite{pinto2017robust, huang2022robust}, domain shift~\cite{muratore2019assessing, ju2022transferring}, and reward perturbations~\cite{wang2020reinforcement,eysenbach2021maximum}.

This work focuses on the policy robustness against observation perturbations, which has been actively researched recently~\cite{fischer2019online, zhang2020robust, oikarinen2021robust, liang2022efficient}.
Several works improve robustness against observation perturbations utilizing various policy regularizers, which enforce the policy to make similar decisions under similar observations~\cite{zhang2020robust,shen2020deep,oikarinen2021robust}.
For instance, Shen et al.~\cite{shen2020deep} design a policy regularizer for continuous control tasks inspired by the Lipschitz continuity, which improves sample efficiency and robustness to adversarial perturbations.
Some researchers attempt to enforce policy robustness utilizing adversarial samples generated through active attacks~\cite{mandlekar2017adversarially,pattanaik2018robust,DBLP:conf/iclr/ZhangCBH21,liang2022efficient}.
Zhang et al.~\cite{DBLP:conf/iclr/ZhangCBH21} propose ATLA, which improves the policy robustness in continuous control tasks by training the policy with an RL adversary online together.
However, Korkmaz~\cite{korkmaz2021investigating,Korkmaz_2023} points out that adversarially trained DRL policies may still be sensitive to policy-independent perturbations.
Several researchers study the certified robustness of DRL policies~\cite{fischer2019online,everett2021certifiable}.
Some methods such as CROP~\cite{wu2022crop} and Policy Smoothing~\cite{kumar2022policy} are proposed to analyze robustness certificates for trained DRL policies.

Despite the significant achievements, there are still some limitations to be addressed.
For instance, they may suffer from high computational costs~\cite{DBLP:conf/iclr/ZhangCBH21} and struggle to cope with strong perturbations, such as perturbations strength greater than ${5}/{255}$ in video games~\cite{wu2022robust}.
In this work, we propose a new robust RL method called \emph{SortRL}.
To our knowledge, this is the first method to improve the robustness of RL policies against observation perturbations from the perspective of network architecture.

\subsection{Robustness of Neural Networks}
Standard neural networks are vulnerable to small perturbations to the inputs~\cite{DBLP:journals/corr/SzegedyZSBEGF13,madry2018towards}, especially given high dimensional inputs such as images.
In order to improve the robustness of DNN, various methods are proposed, including randomized smoothing~\cite{salman2019provably} and relaxation-based approaches~\cite{gowal2018effectiveness,zhang2020towards}.
Besides, some researchers have found that the Lipschitz continuity is significant to the network robustness~\cite{tsuzuku2018lipschitz,anil2019sorting,li2019preventing}.
Recently, several Lipschitz Neural Networks (LNN) have been proposed to enhance robustness, including Spectral Norm~\cite{gouk2021regularisation}, GroupSort~\cite{anil2019sorting}, and $l_\infty$-distance neuron~\cite{zhang2022boosting, zhang2021towards}.
In this work, we construct the policy network based on an $l_\infty$ $1$-Lipschitz Neural Network called SortNet~\cite{zhang2022rethinking}, which provides Lipschitz property, strong expressive power, and high computation efficiency.

\section{Methodology}

\subsection{Problem Formulation}
To study policy robustness under observation perturbations, we formulate the decision process based on the state-adversarial Markov Decision Process (SA-MDP)~\cite{zhang2020robust}.
In this work,  an SA-MDP $\widetilde{\mathcal{M}}$ is defined as $<\mathcal{S}, \mathcal{A}, {P}, {R}, \gamma, \rho, \nu>$, where $\mathcal{S}$ is the state space, $\mathcal{A}$ denotes the action space, 
$P(s'|s,a)=\Pr(s_{t+1}=s'|s_{t}=s,a_t=a)$ denotes the transition probability,
$R:\mathcal{S}\times \mathcal{A}\times\mathcal{S}\to \mathbb{R}$ denotes the reward function, $\gamma \in [0,1]$ denotes discount factor, and $\rho(s)=\Pr(s_0)$ is the distribution of initial states.
$\pi:\mathcal{S}\to \Pr(\mathcal{A})$ is a stationary policy, which is trained to maximize the cumulative reward.

Different from typical MDP $\mathcal{M}$, there exists an adversary $\nu(s):\mathcal{S}\to \Pr(\mathcal{S})$ in SA-MDP $\widetilde{\mathcal{M}}$, which adds perturbations to the agent's observations.
Each time the agent obtains perturbed observation $\hat{s}\sim \nu(s)$ and makes the decision $a\sim\pi(\cdot|\hat{s})$.
Therefore, the value function of policy $\pi$ under $\nu$ adversary is given as follows:
\begin{equation}
\widetilde{V}_{\pi\circ\nu}(s) = \mathbb{E}_{\hat{s}_t\sim\nu(s_t), a_t\sim\pi(\hat{s}_t)}\left[ 
\sum_{t=0}^{\infty} \gamma^t r_{t+1} | s_0 = s
\right].
\end{equation}
In this work, we focus on RL tasks with discrete action spaces against $l_\infty$ bounded perturbations, i.e. 
$\nu(s)\in \mathcal{B}_{\epsilon}^{\infty}(s)$, where $\mathcal{B}_{\epsilon}^{\infty}(s) = \{\hat{s}| \, \|\hat{s}-s\|_{\infty} \leq \epsilon\}$ denotes the ``neighbors" of the clean state $s$.
The $\epsilon \geq 0$ is an important parameter determining the strength of the adversary.
A larger value of $\epsilon$ indicates a stronger adversary, which in turn requires a higher level of policy robustness.
Thus, the policy $\pi$ can be trained by solving the following optimization problem:
\begin{equation}
\label{eq:optim_problem}
\begin{aligned}
& \max_{\pi} \, \min_{\nu} \, \mathbb{E}_{s\sim \rho} \left[ \widetilde{V}_{\pi\circ\nu}(s) \right] \\
& \begin{array}{r@{\quad}l@{\;}l@{\quad}l}
\text {s.t.}
& \|\hat{s}-s\|_{\infty} \leq \epsilon, & \forall s\in \mathcal{S}, \, \hat{s}\sim\nu(s).\\
\end{array}
\end{aligned}
\end{equation}

\subsection{Problem Transformation}
We are required to solve a minimax optimization problem as described in Eq.~\eqref{eq:optim_problem}.
However, finding the optimal adversary $\nu^*(s) = \arg\min_{\nu} \widetilde{V}_{\pi\circ\nu}(s)$ for each state $s_t$ is NP-hard, which is computationally and sample expensive~\cite{oikarinen2021robust}.
To address this issue, we try to reformulate the problem in this section.

\begin{theorem}
\label{thm:v_to_pi}
Given a typical MDP $\mathcal{M}$, corresponding SA-MDP $\widetilde{\mathcal{M}}$ with an adversary $\nu(s)\in \mathcal{B}_{\epsilon}^{\infty}(s)$, and a policy $\pi$, 
$V_{\pi}(s)$ and $\widetilde{V}_{\pi\circ\nu}(s)$ denote the value functions in $\mathcal{M}$ and $\widetilde{\mathcal{M}}$ accordingly. 
We have:
\begin{equation}
\label{eq:th_v_to_pi}
\begin{aligned}
    & \max_{s\in\mathcal{S}} \{ V_{\pi}(s) - \min_{\nu} \widetilde{V}_{\pi\circ\nu}(s) \} \\
    & \leq \alpha \max_{s\in\mathcal{S}} \max_{\nu} \sqrt{ D_{\operatorname{KL}}(\pi(s), \pi(\hat{s}))},\\
\end{aligned}
\end{equation}
where $\alpha= \sqrt{2} \left[1+\frac{\gamma}{\left(1-\gamma\right)^2}\right] \max_{(s,a,s')}|R(s,a,s')|$ is a constant independent of the policy, $\hat{s}\sim\nu(s)$ denotes perturbed observation, and $D_{\operatorname{KL}}(\cdot, \cdot)$ denotes KL-divergence.
\end{theorem}

The proof is given in Appendix A.1 according to \cite{achiam2017constrained} and \cite{zhang2020robust}.
Theorem~\ref{thm:v_to_pi} indicates that the performance loss of the policy $\pi$ under the optimal adversary $\nu^*$ is bounded by the KL divergence between the action distributions.
Therefore, in order to minimize the performance loss of $\pi$ against the observation adversary, we can minimize the $D_{\operatorname{KL}}$ illustrated in Eq.~\eqref{eq:th_v_to_pi} during training.
One possible approach is constructing policy regularizers based on $D_{\operatorname{KL}}$, such as 
$\mathcal{L}_{\operatorname{KL}} = \mathbb{E}_{s}\left[ \max_{\nu} D_{\operatorname{KL}}\left(\pi\left(s\right), \pi\left(\hat{s}\right)\right) \right],$
which is minimized during training the policy.
However, finding the adversary 
$\arg\max_{\nu}D_{\operatorname{KL}}\left(\pi\left(s\right), \pi\left(\hat{s}\right)\right)$ 
for each state $s$ is still computationally expensive.
Besides, policy regularization may hinder the expressive power of the policy network, resulting in the sacrifice of optimality and performance.
In order to address these issues, we introduce the robust radius of policies and incorporate the Lipschitz continuity into the policy network.

\begin{definition}
\label{def:robust_radius}
    (Robust radius of policies)
    Given a stationary policy $\pi$, the robust radius of $\pi$ at state $s$ is defined as the radius of the largest $l_\infty$ ball centered at $s$, in which $\pi$ does not change its decision.
    The formulation is shown as follows:
    \begin{equation}
    \label{eq:robust_radius_def}
        \mathcal{R}(\pi, s) = \inf_{\substack{\pi(s^\prime)\neq\pi(s), s^\prime \in \mathcal{S}}} \|s^\prime-s\|_\infty.
    \end{equation}
\end{definition}

As described in Definition~\ref{def:robust_radius}, the robust radius of policy $\pi$ is designed to evaluate policy robustness against observation perturbations quantitatively.
We can obtain the following formulation based on Theorem~\ref{thm:v_to_pi}:
\begin{equation}
\label{eq:radius_to_no_v_gap}
\begin{aligned}
    \forall s, \, \mathcal{R}(\pi,s) \geq \epsilon \implies 
     \forall s, \, \min_{\nu} \widetilde{V}_{\pi\circ\nu}(s) \geq V_{\pi}(s), \\
\end{aligned}
\end{equation}
which can be proved utilizing Eq.~\eqref{eq:th_v_to_pi} and Eq.~\eqref{eq:robust_radius_def}.
The detailed proof is given in Appendix A.2.
The Eq.~\eqref{eq:radius_to_no_v_gap} implies that, the policy $\pi$ can resist all attacks from $\nu$ without any degradation in performance when the robust radius is big enough.
Therefore, the original problem Eq.~\eqref{eq:optim_problem} can be reformulated as the following equation:
\begin{equation}
\label{eq:optim_proble_radius}
\begin{aligned}
& \max_{\pi} \, \mathbb{E}_{s\sim \rho} \left[ V_{\pi}(s) \right] \\
& \begin{array}{r@{\quad}l@{\;}l@{\quad}l}
\text {s.t.} 
& \mathcal{R}(\pi,s)\geq \epsilon, & \forall s\in \mathcal{S}.\\
\end{array}
\end{aligned}
\end{equation}
Note that solving problem Eq.~\eqref{eq:optim_proble_radius} removes the requirement of finding the  optimal adversary $\nu^*$ compared to Eq.~\eqref{eq:optim_problem}.

\subsection{SortRL Policy Networks}
The problem described in Eq.~\eqref{eq:optim_proble_radius} involves computing the robust radius of policy $\pi$ accurately. 
However, this task is particularly challenging for typical DNN-based policies due to the high computational cost~\cite{zhai2019macer,zhang2021towards}.
In this section, we design a novel policy network utilizing the architecture called SortNet~\cite{zhang2022rethinking} to address this issue with the Lipschitz property.

We utilize a function $g^\pi:\mathcal{S} \to \mathbb{R}^{| \mathcal{A} |}$ to evaluate the score of each action $a\in \mathcal{A}$ based on the perturbed state $\hat{s}$ obtained by the agent.
$g^\pi$ is composed of $M$-layer fully-connected SortNet~\cite{zhang2022rethinking}.
Given a perturbed state $s$, $\boldsymbol{x}^{(0)} = s$ denotes the input of $g^\pi$, and $\boldsymbol{x}^{(l)}_{k}$ denotes the $k$-th unit in the $l$-th layer, which can be computed through the following formulations:
\begin{equation}
\label{eq:sortnet_layer}
\begin{aligned}
    &x_k^{(l)}=\left(\boldsymbol{w}^{(l, k)}\right)^{\mathrm{T}} \operatorname{sort}\left(\left|\boldsymbol{x}^{(l-1)}+\boldsymbol{b}^{(l, k)}\right|\right), \\
    & \omega^{(l,k)}_i = (1-\rho)\rho^{i-1}, \, 1\leq l\leq M, \, 1\leq k\leq d_l, \\
\end{aligned}
\end{equation}
where $d_l$ is the size of $l$-th network layer, $\rho \in [0,1)$ is a hyper-parameter.
$\operatorname{sort}(\boldsymbol{x}) \coloneqq \left[x_{[1]}, \cdots, x_{[d]}\right]^{\mathrm{T}}$, where $x_{[k]}$ is the $k$-th largest element of $\boldsymbol{x}\in \mathbb{R}^d$.
The final output $g^\pi(s) =  - \left( \boldsymbol{x}^{(M)} + \boldsymbol{b}^{\operatorname{out}}  \right)$.
Afterward, the agent takes the best  action with the highest score:
\begin{equation}
\label{eq:pi_def}
    \pi(a|s) \coloneqq \mathds{1}\left( a = \arg\max_{a_i} g^\pi_i(s) \right),
\end{equation} 
where $\mathds{1}(\cdot)$ denotes the indicator function.
The $\big\{ \boldsymbol{b}^{(l, k)} \big\}$ and $\boldsymbol{b}^{\operatorname{out}}$ are network parameters which need to be optimized during training.

\begin{definition}
    (Lipschitz Continuity) Given a function $f:\mathbb{R}^n\to\mathbb{R}^m$, if $\, \exists K>0$, such that
    \begin{equation}
    \label{eq:lipschitz_def}
        \| f(x_1) - f(x_2) \|_p \leq K \| x_1 - x_2 \|_p, \, \forall x_1, x_2\in \mathbb{R}^n,
    \end{equation}
    then $f$ is called $K$-Lipschitz continuous with respect to $l_{p}$ norm, where $K$ is the Lipschitz constant. 
    Similarly,  a neural network $f:\mathbb{R}^n\to\mathbb{R}^m$ is called $l_\infty$ $1$-Lipschitz Neural Network (LNN) if Eq.~\eqref{eq:lipschitz_def} holds with $p=+\infty$ and $K=1$.
\end{definition}

\begin{proposition}
\label{prop:sortnet_lipschitz}
The score function $g^\pi(s)$ is $1$-Lipschitz continuous with respect to $l_\infty$ norm, i.e. 
\begin{equation}
     \left\| g^\pi(s_1) - g^\pi(s_2) \right\|_{\infty} \leq \left\| s_1 - s_2 \right\|_{\infty}, \, \forall s_1, s_2\in \mathcal{S}.
\end{equation}
\end{proposition}
The detailed proof is given in Appendix A.3.

\begin{theorem}
\label{thm:policy_robust_radius}
Given a SortRL policy $\pi$ described in Eq.~\eqref{eq:pi_def}, the lower bound of the robust radius for $\pi$ can be expressed as follows:
\begin{equation}
    \mathcal{R}(\pi, s)\geq \frac{1}{2}\operatorname{margin}(g^\pi, s) , \, \forall s\in \mathcal{S},
\end{equation}
where $\operatorname{margin}(g^\pi, s)$ denotes the difference between the largest and second-largest action scores output by $g^\pi$ at state $s$.
\end{theorem}
The proof of Theorem~\ref{thm:policy_robust_radius} is given in Appendix A.4.
This theorem indicates that, $\forall s\in\mathcal{S}$, if $\operatorname{margin}(g^\pi, s) \geq 2\epsilon$, we can obtain that $\pi(s)=\pi(\hat{s}), \, \forall \hat{s}\sim \nu(s)$, i.e. the SortRL $\pi$ can resist attacks from any adversary $\nu\in \mathcal{B}_{\epsilon}^{\infty}(s)$.
Therefore, the optimization problem described in Eq.~\eqref{eq:optim_proble_radius} can be transformed as follows:
\begin{equation}
\label{eq:optim_problem_margin}
\begin{aligned}
& \max_{\pi} \, \mathbb{E}_{s\sim \rho} \left[ V_{\pi}(s) \right] \\
& \begin{array}{r@{\quad}l@{\;}l@{\quad}l}
\text {s.t.} 
& \displaystyle \pi(a|s) = \mathds{1}\big( a = \arg\max_{a_i\in\mathcal{A}} g^\pi_i(s) \big), &\\
& \displaystyle \operatorname{margin}(g^\pi,s)\geq 2\epsilon, \, \forall s\in \mathcal{S}.&\\
\end{array}
\end{aligned}
\end{equation}

Fortunately, the margin defined in Theorem~\ref{thm:policy_robust_radius} is easy to calculate and can be directly obtained from the network output.
Thus, it is practical to improve the robustness of policy $\pi$ against observation perturbations by optimizing the margin of $g^\pi$.

\subsection{SortRL Training Framework}
\label{sec:training_framework}
In this section, we design a training framework for the policy network $g^\pi$ to solve the problem illustrated in Eq.~\eqref{eq:optim_problem_margin}.
Different from typical DNNs, the output of each layer in $g^\pi$ is biased (always being non-negative) under random initialization.
The biases of each layer are accumulated, leading to unstable or ineffective outputs of the network, which need to be removed with per-layer normalization, i.e.
$\boldsymbol{x}^{(l)} \leftarrow \boldsymbol{x}^{(l)} - \mathbb{E}\left[\boldsymbol{x}^{(l)}\right]$.
The estimation of $\mathbb{E}\left[\boldsymbol{x}^{(l)}\right]$ is inaccessible due to the distribution drift of input observations  during the training of typical DRL algorithms.
More details are given in Appendix B.2.

To address this issue, we introduce a new training pipeline for $g^\pi$ based on Policy Distillation (PD)~\cite{DBLP:journals/corr/RusuCGDKPMKH15}.
Firstly, given a task modeled as $\widetilde{\mathcal{M}}$, a DNN-based teacher policy $\pi_T$ is trained in the typical MDP $\mathcal{M}$ utilizing arbitrary DRL algorithms, i.e. $\pi_T \leftarrow \arg\max_{\pi} \mathbb{E}_{s\sim \rho} \left[ V_{\pi}(s) \right]$.
An expert dataset $\mathcal{D} \coloneqq \left\{ (s, a^*) \right\}$ is constructed through interaction between the teacher policy $\pi_T$ and the clean environment without adversary, where $s$ and $a^*$ denote the clean states and teacher actions correspondingly, i.e. $a^* = \arg\max_a \pi_T(a|s)$.

Afterward, a SortRL policy $\pi_S$ is constructed as the student policy, which is trained to mimic the decisions of the teacher policy $\pi_T$, while maintaining robustness against perturbations.
In this work, the $\pi_S$ is trained by minimizing the following loss function on the expert dataset $\mathcal{D}$: 
\begin{equation}
\label{eq:pi_loss}
\begin{aligned}
    \mathcal{L}_{\pi_S} = \, &\lambda \, \mathbb{E}_{(s,a^*)\sim \mathcal{D}}\big[\mathcal{L}_{\operatorname{CE}}\left(g^\pi(s), a^*\right)\big] \\
     & + \mathbb{E}_{(s,a^*)\sim \mathcal{D}}\big[ 
    \mathcal{L}_{\operatorname{Rob}}\left(g^\pi(s), \theta, a^*\right)\big], \\
\end{aligned}
\end{equation}
where $\lambda\in \mathbb{R}$ is a hyper-parameter.
As described in Eq.~\eqref{eq:pi_loss}, the $\mathcal{L}_{\operatorname{CE}}(\cdot,\cdot)$ denotes the cross-entropy loss, which is utilized to improve the performance of $\pi_S$ in the typical MDP $\mathcal{M}$ by mimicking the behaviors of the teacher policy $\pi_T$.
The formulation of $\mathcal{L}_{\operatorname{CE}}(\cdot,\cdot)$ is described as follows:
\begin{equation}
\label{eq:CE_loss}
    \mathcal{L}_{\operatorname{CE}}(\boldsymbol{z}, a^*) = \log\left( \sum_i \mathrm{e}^{z_i} \right) - z_{a^*},
\end{equation}
where $\boldsymbol{z}=g^\pi(s)$ denotes the action logits without SoftMax normalization. 
The $\mathcal{L}_{\operatorname{Rob}}$ utilized in Eq.~\eqref{eq:pi_loss} denotes robustness loss designed based on the Hinge loss.
The formulation of $\mathcal{L}_{\operatorname{Rob}}$ is given as follows:
\begin{equation}
\label{eq:rob_loss}
\mathcal{L}_{\operatorname{Rob}}(\boldsymbol{z}, \theta, y) = 
\begin{cases}
0, \quad \displaystyle z_y < \max_i z_i \; \text{or} \; z_y - \max_{i\neq y} z_i > \theta, \\
\displaystyle\max_{i\neq y} z_i  - z_y, \quad \text{Otherwise}, \\
\end{cases}
\end{equation} 
where $\theta\in \mathbb{R}^{+}$ is the hinge threshold hyper-parameter. 
Decisions made by $\pi_S$  with margins exceeding $\theta$, or deviating from $\pi_T$, are excluded from the robustness training.
As shown in Eq.~\eqref{eq:pi_loss}, the $\mathcal{L}_{\operatorname{Rob}}(\boldsymbol{z}, \theta, y)$ is utilized to improve policy robustness by optimizing the $\operatorname{margin}(g^{\pi}, s)$ to satisfy the requirement described in Eq.~\eqref{eq:optim_proble_radius} and Eq.~\eqref{eq:optim_problem_margin}, i.e. $\mathcal{R}(\pi,s) \geq \frac{1}{2}\operatorname{margin}(\pi, s)\geq \epsilon$.

The parameter $\lambda$ balances between $\mathcal{L}_{\operatorname{CE}}$ and $\mathcal{L}_{\operatorname{Rob}}$, corresponding to the trade-off between optimality (nominal performance of $\pi_S$ in typical $\mathcal{M}$) and robustness (performance against observation perturbations in $\widetilde{\mathcal{M}}$).
During the training process, the value of $\lambda$ is slowly decayed to achieve optimal performance. 
Initially, we mainly focus on minimizing $\mathcal{L}_{\operatorname{CE}}$ of $\pi_{S}$ to learn the decision-making process of $\pi_T$.
In the later stages, a smaller value of $\lambda$ is used to prioritize policy robustness against observation perturbations.
More details including the pseudocode are given in Appendix B.

\section{Experiment}
In this section, to evaluate the performance of our method compared to the existing methods, we conduct experiments on the following three tasks:
\begin{enumerate}[a)]
    \item \textbf{Classic Control:} 
    Experiments on four classic control tasks~\cite{brockman2016openai} are conducted under different perturbation strength, which aim to demonstrate that SortRL improves the robustness of typical DRL policies.
    \item \textbf{Video Games:} Afterward, we compare SortRL with existing robust RL methods on six video games against adversarial perturbations with $0\leq\epsilon\leq\frac{5}{255}$.
    The purpose is to evaluate the robustness of each method against perturbations on high-dimension observations.
    \item \textbf{Video Games with Stronger Adversaries:} In order to evaluate the performance of our method against stronger perturbations, we conduct experiments on video games under adversaries with large strength $\epsilon>\frac{5}{255}$, which is quite challenging and rarely studied in previous works~\cite{wu2022robust}.
    
\end{enumerate}

In this work, all SortRL policies are trained with \emph{AdamW} optimizer~\cite{loshchilov2018decoupled} on a single NVIDIA RTX 3090 GPU.

\subsection{Classic Control}
\subsubsection{Experimental Settings.}
In this experiment, four environments are utilized, including \emph{CartPole}, \emph{Acrobot}, \emph{MountainCar}, and \emph{LunarLander}.
The policy trained by PPO~\cite{schulman2017proximal} algorithm is utilized as the teacher $\pi_T$.
The dataset $\mathcal{D}$ is constructed utilizing $\pi_{T}$ with 50K states and corresponding teacher actions.
The Projected Gradient Descent (PGD)~\cite{madry2018towards} attacker is applied as the adversary $\nu$ in this experiment.
In each step, the observation is perturbed with untargeted $l_\infty$ PGD attacks with 10 steps.
Each method is evaluated with different perturbation strength $\epsilon \in [0.0, 0.2]$, and the episode rewards are recorded to evaluate robustness.

\subsubsection{Results and Analysis.}
The experiment results are given in Fig.~\ref{fig:classic_control_res}, where $x$-axis denotes $\epsilon$ value and $y$-axis denotes episode rewards under perturbations.
The mean episode rewards and standard errors are given at $\epsilon$ intervals of $0.02$, corresponding to curves and shades respectively.

As shown in Fig.~\ref{fig:classic_control_res}, our method SortRL (orange) outperforms PPO (blue) with higher rewards generally, especially on tasks with large perturbation strength $\epsilon>0.1$.
The episode rewards of both methods decrease as $\epsilon$ increases, but SortRL decays much slower than PPO expert, which demonstrates better robustness of our method.
Besides, in some nominal tasks ($\epsilon=0.0$), there exists a small performance loss of our method compared to PPO, such as \emph{MountainCar} and \emph{LunarLander}.
This performance gap is reported and discussed in the previous studies~\cite{liang2022efficient}.
One possible explanation is that the robustness loss $\mathcal{L}_{\operatorname{Rob}}$ encourages the policy to become smoother and may harm the expressive power to some extent, which is necessary and crucial for nominal performance.

\begin{figure}[t]
\centering
\hspace{-4.2mm}
\subfigure[\emph{CartPole}]{
\includegraphics[width=0.495\linewidth]{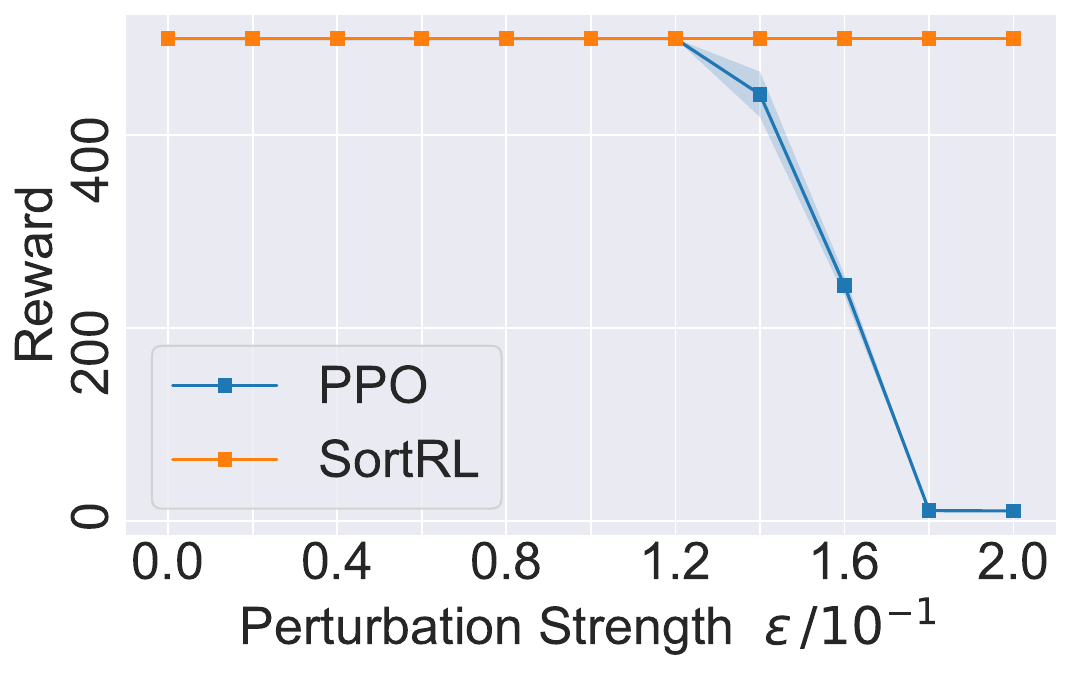}
\label{fig:cartpole_res}
}
\hspace{-4.2mm}
\subfigure[\emph{Acrobot}]{
\includegraphics[width=0.495\linewidth]{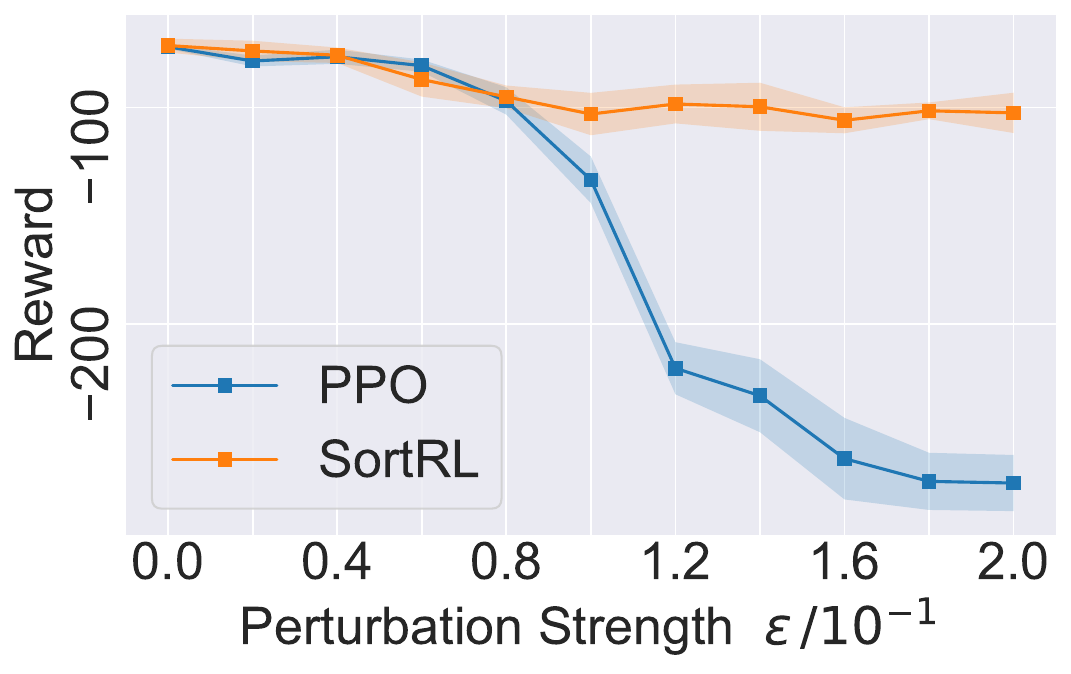}
\label{figure:acrobot_res}
}
\\
\hspace{-4.2mm}
\subfigure[\emph{MountainCar}]{
\includegraphics[width=0.495\linewidth]{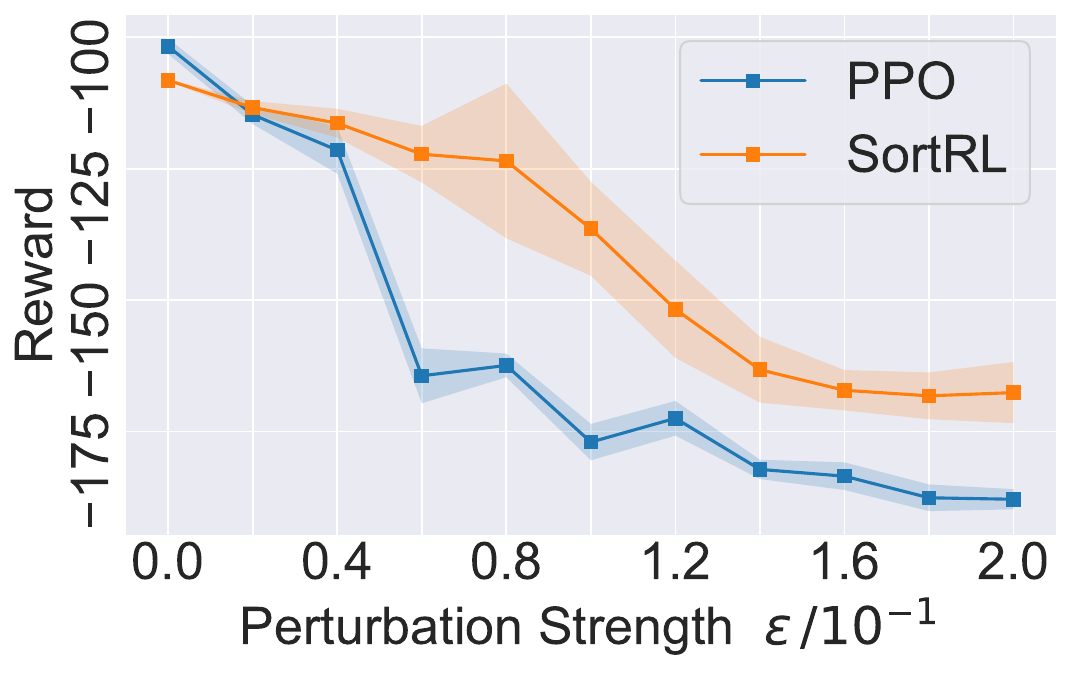}
\label{fig:moutaincar_res}
}
\hspace{-4.2mm}
\subfigure[\emph{LunarLander}]{
\includegraphics[width=0.495\linewidth]{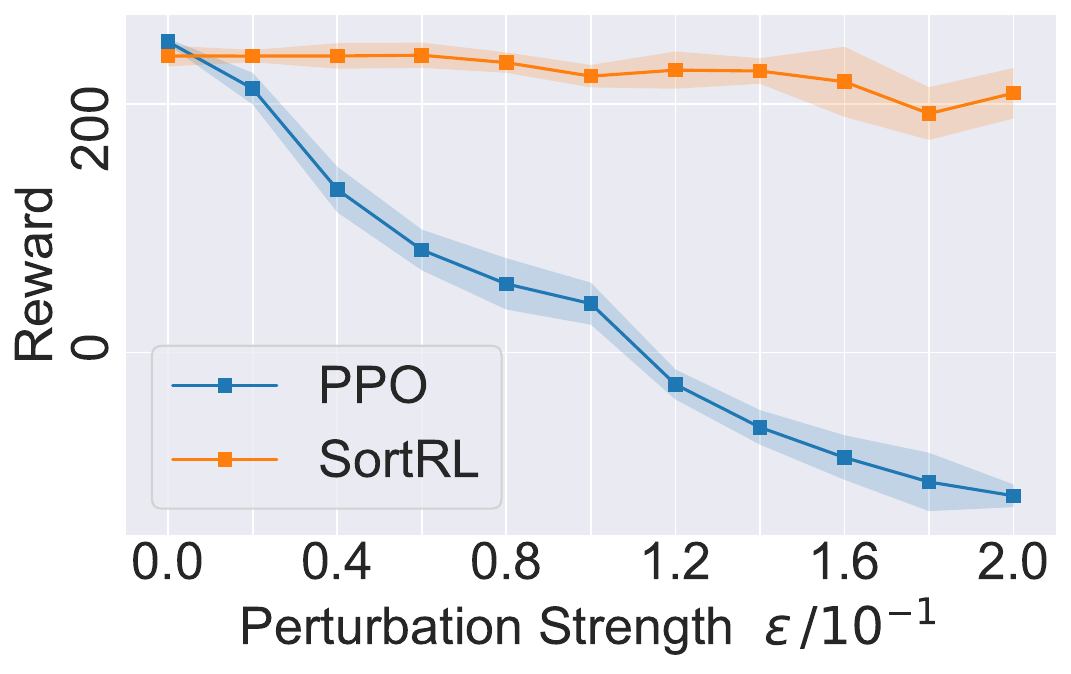}
\label{figure:lunarlander_res}
}
\caption{The experiment results on the classic control tasks.
}
\label{fig:classic_control_res}
\end{figure}

\subsection{Video Games}
\label{sec:expe_video_games}

\subsubsection{Experimental Settings.}
In this experiment, we utilize six video games listed as follows. 
Atari tasks~\cite{bellemare2013arcade}: \emph{Freeway}, \emph{RoadRunner}, \emph{Pong}, and \emph{BankHeist}.
ProcGen tasks~\cite{cobbe2020leveraging}: \emph{Jumper} and \emph{Coinrun}.
Teacher policies are constructed utilizing  DQN~\cite{mnih2015human} and PPO~\cite{schulman2017proximal} for Atari and ProcGen tasks accordingly.
The dataset $\mathcal{D}$ is composed of 100k states and corresponding teacher actions.
To evaluate the robustness,  $l_\infty$-PGD attackers with 10 steps and different strength $\epsilon\in\left\{ \frac{1}{255}, \frac{3}{255}, \frac{5}{255} \right\}$  are applied as the adversary $\nu$ in this experiment.
In each frame, the adversary $\nu$ performs untargeted  attacks on the input observation, which cheats the policy to change decisions.

\subsubsection{Baselines.}
We compare SortRL with the following representative methods: 
(1) \emph{Standard DRL algorithms}, including DQN~\cite{mnih2015human}, A3C~\cite{mnih2016asynchronous}, and PPO~\cite{schulman2017proximal}.
(2) \emph{RS-DQN}~\cite{fischer2019online} designed with adversarial training and provably robust training.
(3) \emph{SA-DQN}~\cite{zhang2020robust} regularizing policy networks based on convex relaxation.
(4) \emph{WocalR}~\cite{liang2022efficient}, which estimates and optimizes the worst-case reward of the policy network under bounded attacks.
(5) \emph{RADIAL}~\cite{oikarinen2021robust}, which trains policy networks by adversarial loss functions based on robustness bounds.

\subsubsection{Evaluation Metrics.}
(1) The episode reward against 10 steps PGD perturbations with $\epsilon\in \left\{\frac{1}{255},\frac{3}{255},\frac{5}{255} \right\}$, which is widely used in previous works~\cite{fischer2019online,zhang2020robust,oikarinen2021robust}.
(2) Action Certification Rate (ACR)~\cite{zhang2020robust}, which is designed to evaluate policy performance on certified robustness.
ACR is defined as the proportion of the actions during rollout that are guaranteed unchanged with any adversary $\nu\in\mathcal{B}_{\epsilon}^{\infty}$.
The detailed computation process of ACR for SortRL is given in Appendix D.2.

\subsubsection{Results and Analysis.}

\begin{table*}[t]
\centering
\begin{tabular}{c|l|llll|l}
\hline
\multicolumn{1}{l|}{\textbf{Task}} & {\textbf{Model/Metric}} & \multicolumn{4}{c|}{\textbf{Episode Reward}} & \textbf{ACR ($\%$)} \\ \hline
\multicolumn{1}{l|}{} & {$\epsilon$} & $0$ (nominal) & $1/255$ & $3/255$ & $5/255$ & $1/255$ \\ \hline
\multirow{6}{*}{Freeway} & {DQN} & \underline{$33.9\pm0.07$} & $0.0\pm0.0$ & $0.0\pm0.0$ & $0.0\pm0.0$ & $0.0$ \\
 & {RS-DQN} & $32.93$ & $32.53$ & N/A & N/A & N/A\\
 & {SA-DQN} & $30.0\pm0.0$ & $30.0\pm0.0$ & $30.05\pm0.05$ & $27.65\pm0.22$ & \bm{$100.0$}\\
 & {WocaR-DQN} & $31.2\pm0.4$ & $31.2\pm0.5$ & $31.4\pm0.3$ & $21.1\pm1.75$ & $99.90$\\ 
 & {RADIAL-DQN} & $33.2\pm0.19$ & \underline{$33.35\pm0.16$} & \underline{$33.4\pm0.13$} & \underline{$29.1\pm0.17$} & $99.82 $\\ \cline{2-7}  
 \rowcolor{Gray}\cellcolor{white} &{SortRL-DQN} & \bm{$33.91\pm0.32$} & \bm{$33.83\pm0.48$} & \bm{$33.94\pm0.24$} & \bm{$33.92\pm0.33$} & $99.94$ \\ \hline
\multirow{8}{*}{\begin{tabular}[c]{@{}c@{}}Road\\ Runner\end{tabular}} & {DQN} & $43390\pm973$ & $0.0\pm0.0$ & $0.0\pm0.0$ & $0.0\pm0.0$ & $0.0$ \\
 & {A3C} & $34420\pm604$ & $31040\pm2173$ & $3025\pm317$ & $350\pm93$ & $0.0$ \\
 & {RS-DQN} & $12106.67$ & $5753.33$ & N/A & N/A & N/A \\
 & {SA-DQN} & \bm{$45870\pm1380$} & $44300\pm1753$ & $20170\pm1822$ & $3350\pm335$ & $60.20$ \\
 & {RADIAL-DQN} & \underline{$44495\pm1165$} & \underline{$44445\pm1148$} & \underline{$39560\pm1621$} & $23820\pm942$ & $99.42$\\
 & {WocaR-DQN} & $44156\pm2279$ & $44079\pm2154$ & $38720\pm1765$ & $ 3490\pm1959 $ & $98.41$\\
 & {RADIAL-A3C} & $34825\pm981$ & $31960\pm933$ & $29920\pm1496$ & \underline{$31545\pm1480$} & $92.33$\\  \cline{2-7} 
 \rowcolor{Gray}\cellcolor{white} &{SortRL-DQN} & $43697\pm1457$ & \bm{$44596\pm1070$} & \bm{$39766\pm1176$} & \bm{$40905\pm1249$} & \bm{$99.98$} \\ \hline
 & {DQN} & \bm{$21.0\pm0.0$} & $-21.0\pm0.0$ & $-21.0\pm0.0$ & $-20.85\pm0.08$ & $0.0$ \\
 & {A3C} & \bm{$21.0\pm0.0$} & \bm{$21.0\pm0.0$} & \bm{$21.0\pm0.0$} & $-17.85\pm0.33$ & $0.0$\\
 & {RS-DQN} & $19.73$ & $18.13$ & N/A & N/A & N/A \\
 & {SA-DQN} & \bm{$21.0\pm0.0$} & \bm{$21.0\pm0.0$} & \bm{$21.0\pm0.0$} & $-19.75\pm0.1$ & \bm{$100.0$}\\
 & {WocaR-DQN} & \bm{$21.0\pm0.0$} & \bm{$21.0\pm0.0$} & \bm{$21.0\pm0.0$} & $-20.7\pm 0.45$ & $59.05$\\ 
 \rowcolor{Gray}\cellcolor{white} & {RADIAL-DQN} & \bm{$21.0\pm0.0$} & \bm{$21.0\pm0.0$} & \bm{$21.0\pm0.0$} & \bm{$21.0\pm0.0$} & $89.49$ \\
 \rowcolor{Gray}\cellcolor{white} &{RADIAL-A3C} & \bm{$21.0\pm0.0$} & \bm{$21.0\pm0.0$} & \bm{$21.0\pm0.0$} & \bm{$21.0\pm0.0$} & $75.53$\\ \cline{2-7} 
 \rowcolor{Gray}\cellcolor{white} \multirow{-8}{*}{Pong}  & {SortRL-DQN} & \bm{$21.0\pm0.0$} & \bm{$21.0\pm0.0$} & \bm{$21.0\pm0.0$} & \bm{$21.0\pm0.0$} & \bm{$100.0$}\\ \hline
\multirow{9}{*}{\begin{tabular}[c]{@{}c@{}}Bank\\ Heist\end{tabular}} & {DQN} & $1325.5\pm5.7$ & $29.5\pm2.4$ & $0.0\pm0.0$ & $0.0\pm0.0$ & $0.0$\\
 & {A3C} & $1109.0\pm21.4$ & $1102.5\pm49.4$ & $534.5\pm58.2$ & $115.0\pm27.8$ & $0.0$\\
 & {RS-DQN} & $238.66$ & $190.67$ & N/A & N/A & N/A \\
 & {SA-DQN} & $1237.6\pm1.7$ & $1237.0\pm2.0$ & $1213.0\pm2.5$ & $1130.0\pm29.1$ & $97.63$ \\
 & {WocaR-DQN} & $1220\pm12$ & $1220\pm3$ & $1214\pm7$ & $ 1094\pm 20 $ & $96.75$ \\
 & {RADIAL-DQN} & \bm{$1349.5\pm1.7$} & \bm{$1349.5\pm1.7$} & \bm{$1348\pm1.7$} & $1182.5\pm43.3$ & $98.17$\\ 
 & {RADIAL-A3C} & $1036.5\pm23.4$ & $975\pm22.2$ & $949\pm19.5$ & $712\pm46.4$ & $71.84$ \\  \cline{2-7} 
 & {SortRL-DQN} & $1323.8\pm6.9$ & $1325.6\pm6.5$ & $1315.1\pm5.8$ & \underline{$1317.8\pm7.2$} & $99.69$ \\
 \rowcolor{Gray}\cellcolor{white} & {SortRL-RADIAL} & \underline{$1342.8\pm5.5$} & \underline{1340.8$\pm4.7$} & \underline{$1345.2\pm4.9$} & \bm{$1341.1\pm5.1$} & \bm{$99.93$} \\ \hline
\end{tabular}
\caption{The experiment results on the Atari video games.
The best results are \textbf{boldfaced}, while the second best ones are \underline{underlined}.
The \colorbox{Gray}{{gray row}} denotes the most robust method, selected based on the score $R_{\epsilon=0} + \frac{1}{3}\sum_\epsilon R_{\epsilon}$, where $R_{\epsilon}$ is the mean episode reward given perturbation strength $\epsilon$.
N/A denotes the authors have not released results, codes, or models.
}
\label{tab:atari_res}
\end{table*}

The experiment result are shown in Table~\ref{tab:atari_res} (Atari) and Table~\ref{tab:procgen_res} (ProcGen).
As shown in the tables, SortRL outperforms baseline methods and achieves higher episode rewards on video game tasks with different perturbation strength, demonstrating the effectiveness of our approach.
Take \emph{RoadRunner} with $\epsilon=5/255$ as an instance, SortRL achieves an episode reward of $40905$ and outperforms existing state-of-the-art $31545$ by $29.6\%$.

\begin{table*}[t]
\centering
{%
\begin{tabular}{c|l|l|llll}
\hline
\textbf{Task} & \multicolumn{2}{c|}{\textbf{Model/Metric}} & \multicolumn{4}{c}{\textbf{Episode Reward}} \\ \hline
\multicolumn{1}{l|}{} & \multicolumn{1}{l|}{$\epsilon$} & Env. Type & $0$ (nominal) & $1/255$ & $3/255$ & $5/255$ \\ \hline
 \multirow{6}{*}{Jumper} & \multicolumn{1}{l|}{\multirow{2}{*}{PPO}} & Train & $8.69\pm0.11$ & $6.61\pm0.15$ & $4.50\pm0.16$ & $3.42\pm0.15$ \\
 & \multicolumn{1}{l|}{} & Eval & $4.22\pm0.16$ & $3.90\pm0.15$ & $3.10\pm0.15$ & $3.15\pm0.15$ \\ \cline{2-7} 
 & \multicolumn{1}{l|}{\multirow{2}{*}{RADIAL-PPO}} & Train & $6.59\pm0.15$ & $6.70\pm0.15$ & $6.55\pm0.15$ & $6.83\pm0.15$ \\
 & \multicolumn{1}{l|}{} & Eval & $3.85\pm0.15$ & $3.93\pm0.15$ & $3.75\pm0.15$ & $3.59\pm0.15$ \\ \cline{2-7} 
 \rowcolor{Gray}\cellcolor{white} & & Train & $\bm{9.10\pm0.28}$ & \bm{$9.10\pm0.29$} & \bm{$9.10\pm0.29$} & \bm{$9.10\pm0.29$} \\ 
 \rowcolor{Gray}\cellcolor{white} &  {\multirow{-2}{*}{SortRL-PPO (Ours)}} & Eval & \bm{$4.65\pm0.39$} & \bm{$4.63\pm0.39$} & \bm{$4.68\pm0.39$} & \bm{$4.65\pm0.39$} \\ \hline
\multirow{6}{*}{Coinrun} & \multicolumn{1}{l|}{\multirow{2}{*}{PPO}} & Train & $8.31\pm0.12$ & $6.36\pm0.15$ & $4.19\pm0.16$ & $3.32\pm0.15$ \\
 & \multicolumn{1}{l|}{} & Eval & $6.65\pm0.15$ & $5.22\pm0.16$ & $3.58\pm0.15$ & $3.36\pm0.15$ \\ \cline{2-7} 
 & \multicolumn{1}{l|}{\multirow{2}{*}{RADIAL-PPO}} & Train & $7.12\pm0.14$ & $7.10\pm0.14$ & $7.19\pm0.14$ & $7.34\pm0.14$ \\
 & \multicolumn{1}{l|}{} & Eval & $6.66\pm0.15$ & $6.71\pm0.15$ & $6.71\pm0.15$ & $6.67\pm0.15$ \\ \cline{2-7} 
 \rowcolor{Gray}\cellcolor{white} & & Train & \bm{$8.40\pm0.37$} & \bm{$8.51\pm0.36$} & \bm{$8.60\pm0.35$} & \bm{$8.41\pm0.37$} \\
  \rowcolor{Gray}\cellcolor{white} & {\multirow{-2}{*}{SortRL-PPO (Ours)}} & Eval & \bm{$7.33\pm0.44$} & \bm{$7.70\pm0.42$} & \bm{$7.23\pm0.45$} & \bm{$7.20\pm0.41$} \\ \hline
\end{tabular}%
}
\caption{The experiment results on the ProcGen video games.}
\label{tab:procgen_res}
\end{table*}

\begin{figure}[ht]
\centering
\includegraphics[width=0.999\linewidth]{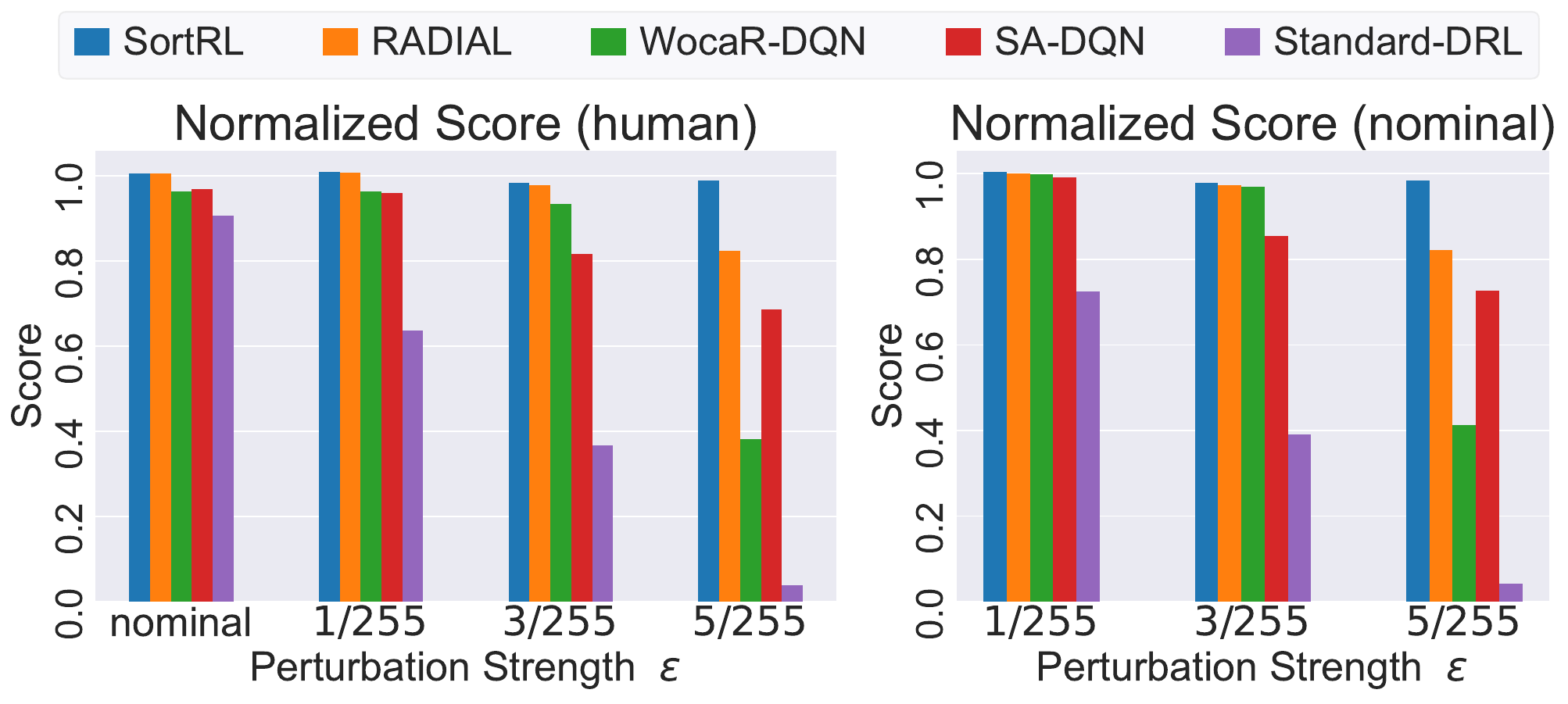}
\caption{Normalized score on Atari tasks.
\textbf{Left:} relative to the human expert.
\textbf{Right:} relative to nominal performance.
}
\label{fig:norm_atari_all}
\end{figure}

As shown in Fig.~\ref{fig:norm_atari_all}, to measure and analyze the performance, we adopt the metric of average normalized score to aggregate episode rewards across tasks.
In detail, given the episode reward $Z$, its normalized score is defined as $\frac{Z-Z_0}{Z_1-Z_0}\in[0,1]$, where $Z_0$ denotes the reward of the random policy, and $Z_1$ denotes human reward or nominal reward.
As described in Fig.~\ref{fig:norm_atari_all}, the advantage of SortRL over baseline methods increases as $\epsilon$ increases generally.
As described in the right figure, compared to the corresponding nominal performance, our method only loses performance less than $1.7\%$ against $\epsilon = {5}/{255}$, while the state-of-the-art RADIAL loses about $18\%$.
Besides, the performance of SortRL on ACR in Table~\ref{tab:atari_res} is also excellent, which is greater than $99.6\%$ in various tasks.
These results demonstrate that, compared to existing methods,  SortRL achieves
policy robustness with fewer sacrifices on the optimality and expressive power of the policy network.
This relies on the Lipschitz property at the network level and the maximization of the robust radius in the training framework. 

It is interesting that SortRL outperforms standard DRL methods in some nominal tasks, such as \emph{Freeway}, \emph{Jumper}, and \emph{Coinrun}.
This implies that robust training with suitable parameter settings may improve nominal performance in some tasks.
Besides, some methods achieve better performance against perturbations than that in nominal environments, such as in the \emph{Coinrun} task with both \emph{RADIAL} and \emph{SortRL} methods.
These interesting phenomena are also observed in previous studies~\cite{zhang2020robust,oikarinen2021robust}.
One possible explanation is that most tasks prefer smooth policies, i.e. similar decisions given similar observations.
However, policies trained by standard DRL are suboptimal due to the non-smoothness property of the policy network, especially in tasks with high-dimension observations, such as ProcGen.
Smoother policies and trajectories with higher rewards may be found through robust training or by adding perturbations to the observations. 

\subsection{Video Games with Stronger Adversaries}

\begin{figure}[tb]
\centering
\hspace{-3.5mm}
\subfigure[\emph{BankHeist}]{
\includegraphics[height=3.06cm]{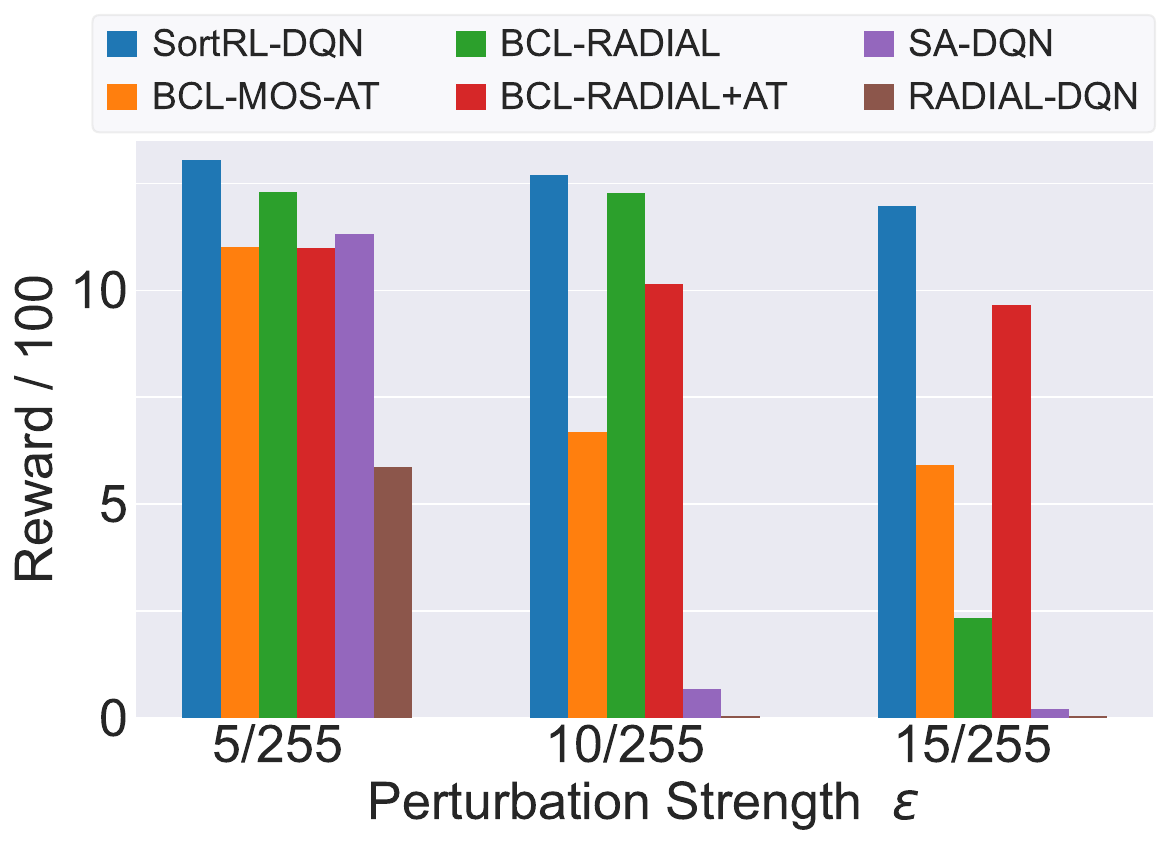}
\label{fig:bankheist_big_eps_res}
} 
\hspace{-3.5mm}
\subfigure[\emph{Freeway}]{
\includegraphics[height=3.06cm]{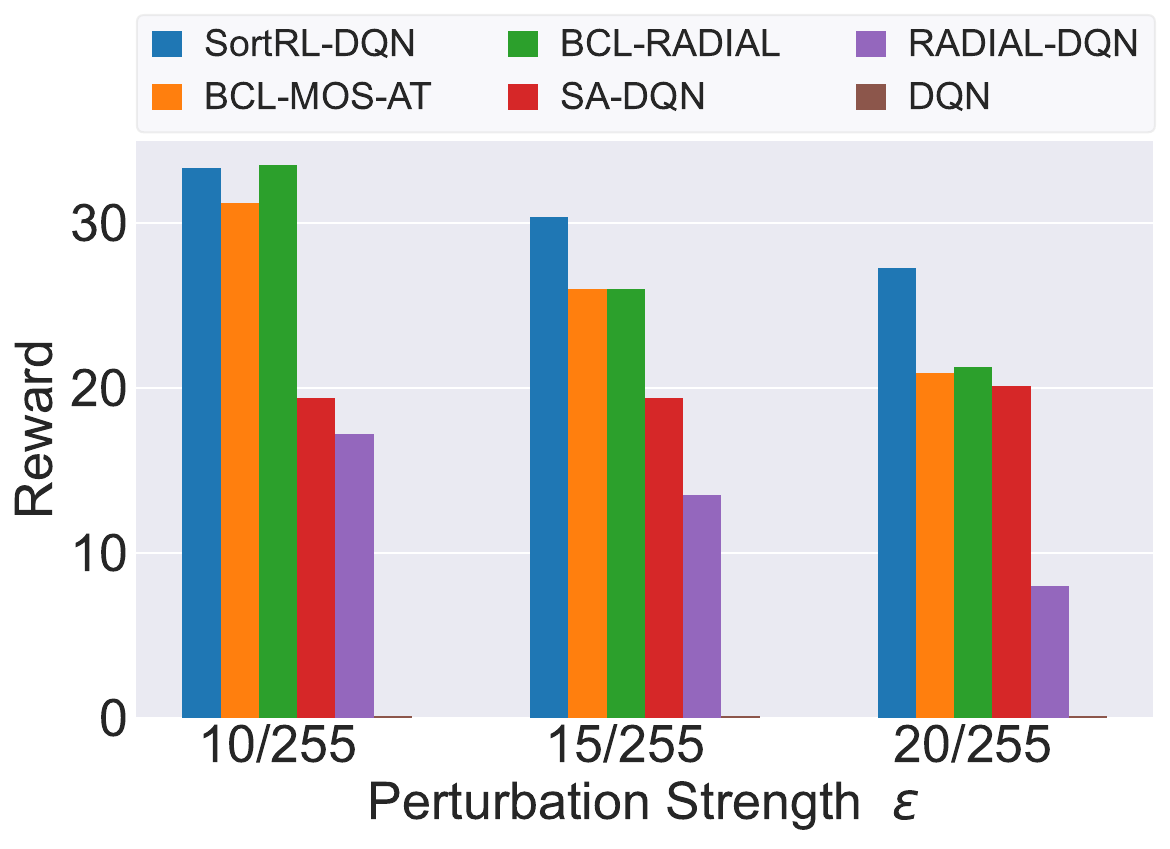}
\label{fig:freeway_big_eps_res}
}
\caption{Experiment results on video games \emph{BankHeist} and \emph{Freeway} with stronger adversaries ($\epsilon\geq5/255$).
}
\label{fig:video_games_big_eps_res}
\end{figure}

\subsubsection{Experimental Settings}
In this experiment, the same teacher policies and dataset $\mathcal{D}$ described in Sec.~\ref{sec:expe_video_games}  are utilized.
In order to evaluate policy robustness against stronger perturbations, we utilize larger strength $\epsilon > {5}/{255}$.
Besides, more attackers~\cite{wu2022robust} are utilized as adversaries $\nu$ in this experiment:
(1) PGD attacker with 30 steps. (2) FGSM attacks with Random Initialization (RI-FGSM)~\cite{wong2019fast} (3) RI-FGSM-Multi: sample multiple random starts for RI-FGSM, and choose the first sample which alters the policy decision (4) RI-FGSM-Multi-T: sample multiple random starts for RI-FGSM, and choose the sample which minimizes the estimated Q values among the samples.

\subsubsection{Baselines and Evaluation Metrics}
Most baseline methods in this section are the same as Sec.~\ref{sec:expe_video_games}, including:
(1) \emph{SA-DQN}, (2) \emph{RADIAL}.
In addition, a new benchmark (3) \emph{Bootstrapped Opportunistic Adversarial Curriculum Learning (BCL)}~\cite{wu2022robust} is utilized, which can enhance the robustness of existing robust RL methods under strong adversaries.
\emph{BCL} is an adversarial curriculum training framework, and can be combined with various robust RL methods, such as \emph{BCL-RADIAL} and \emph{BCL-MOS-AT}.

\subsubsection{Results and Analysis.}

The experiment results on the \emph{BankHeist} task with $\epsilon\in \left\{\frac{5}{255},\frac{10}{255},\frac{15}{255} \right\}$ and \emph{Freeway} task with $\epsilon\in \left\{\frac{10}{255},\frac{15}{255},\frac{20}{255} \right\}$ are illustrated in Fig.~\ref{fig:video_games_big_eps_res}.
As shown in the figures, the $x$-axis denotes the $\epsilon$ value while the $y$-axis denotes episode reward.
More experiment results are given in Appendix D.3.

As shown in Fig.~\ref{fig:video_games_big_eps_res},
SortRL achieves state-of-the-art performance compared to existing methods, especially in tasks with strong perturbations with $\epsilon \geq 15/255$.
Take the \emph{Freeway} task with $\epsilon={20}/{255}$ as an instance, SortRL achieves an episode reward of $27.2$ and outperforms existing state-of-the-art BCL-RADIAL ($21.2$) by approximately $28.3\%$.
The results demonstrate the excellent robustness of SortRL under strong perturbation strength, which relies on the Lipschitz continuity of the policy network and the robust training framework.

\section{Conclusion}
In this work, we propose a novel robust RL method called \emph{SortRL}, which improves the robustness of DRL policies against observation perturbations from the perspective of network architecture.
We employ a new policy network based on Lipschitz Neural Networks, and provide a convenient approach to optimizing policy robustness based on the output margin.
To facilitate training, we design a training framework based on Policy Distillation, which trains the policy to solve given tasks while maintaining a suitable robust radius.
Several experiments are conducted to evaluate the robustness of our method, including control tasks and video games with different perturbation strength.
The experiment results demonstrate that \emph{SortRL} outperforms existing methods on robustness.

\section*{Acknowledgments}
This work was supported by the National Natural Science Foundation of China (Grant No. 92248303 and No. 62373242), the Shanghai Municipal Science and Technology Major Project (Grant No. 2021SHZDZX0102), and the Fundamental Research Funds for the Central Universities.

\bibliography{aaai24}


\clearpage
\onecolumn
\appendix

\section{Proofs}

\subsection{Proof of Theorem~\ref{thm:v_to_pi}}
\label{sec:pf_thm_v_to_pi}
\begin{theorem*}
Given a typical MDP $\mathcal{M}$, corresponding SA-MDP $\widetilde{\mathcal{M}}$ with an adversary $\nu(s)\in \mathcal{B}_{\epsilon}^{\infty}(s)$, and a policy $\pi$, 
$V_{\pi}(s)$ and $\widetilde{V}_{\pi\circ\nu}(s)$ denote the value functions in $\mathcal{M}$ and $\widetilde{\mathcal{M}}$ accordingly. 
We have:
\begin{equation}
     \max_{s\in\mathcal{S}} \{ V_{\pi}(s) - \min_{\nu} \widetilde{V}_{\pi\circ\nu}(s) \} 
     \leq \alpha \max_{s\in\mathcal{S}} \max_{\nu} \sqrt{ D_{\operatorname{KL}}(\pi(s), \pi(\hat{s}))},\\
\end{equation}
where $\alpha= \sqrt{2} \left[1+\frac{\gamma}{\left(1-\gamma\right)^2}\right] \max_{(s,a,s')}|R(s,a,s')|$ is a constant independent of the policy, $\hat{s}\sim\nu(s)$ denotes perturbed observation, and
$D_{\operatorname{KL}}(\cdot, \cdot)$ denotes KL-divergence.
\end{theorem*}

\begin{proof}
This proof is given according to SA-DQN~\cite{zhang2020robust} based on Constrained Policy Optimization (CPO)~\cite{achiam2017constrained}.
Given any starting state $s$  two policies $\pi$ and $\pi^\prime$, Achiam et al.~\cite{achiam2017constrained} proposed that the upper bound of $V_{\pi}(s_0) - V_{\pi^\prime}(s_0)$ can be expresses as follows:
\begin{equation}
\label{eq_pf:cpo_0}
\begin{aligned}
    V_{\pi}(s_0) - V_{\pi^\prime}(s_0) & \leq \frac{2 \gamma}{(1-\gamma)^2} \max_s\big|\mathbb{E}_{a \sim \pi^{\prime}, s^{\prime} \sim P}\left[R\left(s, a, s^{\prime}\right)\right]\big| \mathbb{E}_{s \sim d_{s_0}^\pi}\left[\mathrm{D}_{T V}\left(\pi(s), \pi^{\prime}(s)\right)\right] \\
    & \quad - \frac{1}{1-\gamma} \mathbb{E}_{s \sim d_{s_0}^\pi, a \sim \pi, s^{\prime} \sim P} \left[\left(\frac{\pi^{\prime}(a | s)}{\pi(a|s)}-1\right) R\left(s, a, s^{\prime}\right)\right], \\ 
\end{aligned}
\end{equation}
where $d_{s_0}^\pi(s)=(1-\gamma) \sum_{t=0}^{\infty} \gamma^t \Pr\left(s_t=s | \pi, s_0\right)$ denotes the discounted future state distribution.
$D_{\operatorname{TV}}$ denotes the total variance distance.
We can obtain that:
\begin{equation}
\label{eq_pf:cpo_1}
\begin{aligned}
    &\max_s\big|\mathbb{E}_{a \sim \pi^{\prime}, s^{\prime} \sim P}\left[R\left(s, a, s^{\prime}\right)\right]\big| \mathbb{E}_{s \sim d_{s_0}^\pi}\left[\mathrm{D}_{T V}\left(\pi(s), \pi^{\prime}(s)\right)\right] \\
    & \leq \frac{1}{\sqrt{2}} \max_{(s,a,s')}\left|R\left(s,a,s^\prime\right)\right| \max_{s} \sqrt{ D_{\operatorname{KL}}\left(\pi(s), \pi\left(s^\prime\right)\right)}, \\ 
\end{aligned}
\end{equation}
where $D_{\operatorname{KL}}$ denotes the Kullback–Leibler divergence.
Besides, we also have:
\begin{equation}
\label{eq_pf:cpo_2}
\begin{aligned}
    - \mathbb{E}_{s \sim d_{s_0}^\pi, a \sim \pi, s^{\prime} \sim P} \left[\left(\frac{\pi^{\prime}(a | s)}{\pi(a | s)}-1\right) R\left(s, a, s^{\prime}\right)\right] &=  \mathbb{E}_{s \sim d_s^\pi}\left[ \sum_a\left[\pi(a| s)-\pi^{\prime}(a| s)\right] \sum_{s^{\prime}} p\left(s^{\prime}| s, a\right) R\left(s, a, s^{\prime}\right) \right]\\
    &\leq \mathbb{E}_{s \sim d_{s_0}^\pi}\left[ \sum_a\big|\pi(a| s)-\pi^{\prime}(a| s)\big| \big|\sum_{s^{\prime}} p\left(s^{\prime}| s, a\right) R\left(s, a, s^{\prime}\right)\big| \right] \\
    & \leq (1-\gamma) \max_{(s,a,s')}\left|R\left(s,a,s^\prime\right)\right| \max_{s} \left\{\sum_a\left|\pi(a | s)-\pi^{\prime}(a | s)\right|\right\} \\
    & \leq \sqrt{2}(1-\gamma) \max_{(s,a,s')}\left|R\left(s,a,s^\prime\right)\right| \max_{s} \sqrt{ D_{\operatorname{KL}}\left(\pi(s), \pi\left(s^\prime\right)\right)}. \\
\end{aligned}
\end{equation}
We let $\pi^\prime(s) \leftarrow \pi(\nu^*(s))$, where $\nu^*(s) = \arg\min_{\nu} \widetilde{V}_{\pi\circ\nu}(s)$ denotes the best adversary.
Based on Eq.~\eqref{eq_pf:cpo_0}, Eq.~\eqref{eq_pf:cpo_1}, and Eq.~\eqref{eq_pf:cpo_2}, we can obtain that:
\begin{equation}
V_{\pi}(s_0) - \widetilde{V}_{\pi\circ\nu^*}(s_0) \leq 
\alpha \max_{s^\prime} \sqrt{ D_{\operatorname{KL}}(\pi(s), \pi(s^\prime))},
\end{equation}
where $\alpha = \sqrt{2} \left[1+\frac{\gamma}{\left(1-\gamma\right)^2}\right] \max_{(s,a,s')}|R(s,a,s')|$.
The proof is completed.

\end{proof}

\subsection{Proof of Eq.~\eqref{eq:radius_to_no_v_gap}}
\label{sec:pf_radius_to_no_v_gap}
\begin{statement}
    $\forall s, \, \mathcal{R}(\pi,s) \geq \epsilon \implies \forall s, \, \min_{\nu} \widetilde{V}_{\pi\circ\nu}(s) \geq V_{\pi}(s)$.
\end{statement}

\begin{proof}
The robust radius of policies is defined as $\mathcal{R}(\pi, s) = \inf_{\pi(\hat{s})\neq\pi(s)} \|\hat{s}-s\|_\infty.$
Thus, if $\mathcal{R}(\pi,s) \geq \epsilon$ holds for $\forall s\in\mathcal{S}$, we can obtain that $\pi(s)=\pi(\hat{s})$, i.e. $D_{\operatorname{KL}}(\pi(s), \pi(\hat{s}))=0$ for $\forall s\in\mathcal{S}$.

Then we can obtain the following formulation based on Theorem 1:
\begin{equation*}
    \max_{s\in\mathcal{S}} \{ V_{\pi}(s) - \min_{\nu} \widetilde{V}_{\pi\circ\nu}(s) \}
    \leq \alpha \max_{s\in\mathcal{S}} \max_{\nu} \sqrt{ D_{\operatorname{KL}}(\pi(s), \pi(\hat{s}))} = 0.
\end{equation*}

Therefore, $\min_{\nu} \widetilde{V}_{\pi\circ\nu}(s) \geq V_{\pi}(s)$ holds for $\forall s\in\mathcal{S}$.
The proof is completed.

\end{proof}

\subsection{Proof of Proposition~\ref{prop:sortnet_lipschitz}}
\label{sec:pf_prop_sortnet_lipschitz}
\begin{proposition*}
    The score function $g^\pi(s)$ is $1$-Lipschitz continuous with respect to $l_\infty$ norm, i.e. 
    \begin{equation*}
         \left\| g^\pi(s_1) - g^\pi(s_2) \right\|_{\infty} \leq \left\| s_1 - s_2 \right\|_{\infty}, \, \forall s_1, s_2\in \mathcal{S}.
    \end{equation*}
\end{proposition*}

\begin{proof}
The proof is inspired by~\cite{zhang2022rethinking}.
The $l$-th layer of the function $g^\pi(s)$ is formulated as follows:
\begin{equation*}
    x_k^{(l)}=\left(\boldsymbol{w}^{(l, k)}\right)^{\mathrm{T}} \operatorname{sort}\left(\left|\boldsymbol{x}^{(l-1)}+\boldsymbol{b}^{(l, k)}\right|\right),
    \omega^{(l,k)}_i = (1-\rho)\rho^{i-1}, \, 1\leq l\leq M, \, 1\leq k\leq d_l,
\end{equation*}
where $\boldsymbol{x}^{(l)}_{k}$ denotes the $k$-th unit in the $l$-th layer, $d_l$ is the size of $l$-th network layer, and $\rho \in [0,1)$ is a hyper-parameter.
\begin{lemma*}
    Given any $l_\infty$ $1$-Lipschitz continuous functions $\phi_1(x)$ and $\phi_2(x)$, the composite function $\phi_1\left(\phi_2\left(x\right)\right)$ is also $l_\infty$ $1$-Lipschitz continuous.
\end{lemma*}
Thus, we can prove the Lipschitz continuity of $g^\pi(x)$ by checking its basic operations.
$g^\pi(x)$ is composed of (a) affine transformations $\phi_1(x)=\boldsymbol{\omega}x + \boldsymbol{b}$, where $\omega^{}_i = (1-\rho)\rho^{i-1}$; (b) element-wise absolute value operation $\phi_2(x)=| x |$; (c) Sort operation $\phi_3(x)=\operatorname{sort}(x) = \left[x_{[1]}, \cdots, x_{[d]}\right]^{\mathrm{T}}$, where $x_{[k]}$ is the $k$-th largest element of $\boldsymbol{x}\in \mathbb{R}^d$.
In the following, we give detailed proof that the above three operations are $l_\infty$ $1$-Lipschitz continuous.

\begin{enumerate}[(a)]
\item Lipschitz of the affine transformation $\phi_1(x)$ is proved with the definition of Lipschitz continuity.
Given $\forall x_1, x_2$, we have:
\begin{equation}
\begin{aligned}
    \| \phi_1(x_1) - \phi_1(x_2) \|_\infty 
    &= \| \boldsymbol{\omega}(x_1 - x_2) \|_\infty \\
    &\leq \|\boldsymbol{\omega}\|_\infty \, \| x_1 - x_2 \|_\infty \\
    &\leq  \| x_1 - x_2 \|_\infty.
\end{aligned}
\end{equation}
Thus, the function $\phi_1(x)$ is $l_\infty$ $1$-Lipschitz continuous.

\item Similarly, the element-wise absolute value operation $\phi_2(x)=| x |$ is also Lipschitz continuous:
\begin{equation}
\begin{aligned}
    \| \phi_2(x_1) - \phi_2(x_2) \|_\infty 
    = \| |x_1| - |x_2| \|_\infty 
    \leq  \| x_1 - x_2 \|_\infty.
\end{aligned}
\end{equation}

\item The sort operation $\phi_3(x)=\operatorname{sort}(x)$ is also Lipschitz continuous.
In order to simplify the proof, we introduce a new function called $\psi^{(i,j)}(x)$, which is defined as follows:
\begin{equation*}
\psi^{(i,j)}_k(x) = 
\begin{cases}
\max\{ x_i, x_j \}, & k=i, \\
\min\{ x_i, x_j \}, & k=j, \\
x_k,  & \text{Otherwise}, \\
\end{cases}
\end{equation*}
where $x\in\mathbb{R}^{d}$, $1\leq i < j \leq d$.
As shown in the equation, $\psi^{(i,j)}(x)$ swaps $x_i$ and $x_j$ if $x_i < x_j, i < j$.
The $\phi_3(x)=\operatorname{sort}(x)$ can be implemented with $\psi^{(i,j)}$ functions on $x$ for all $1\leq i < j \leq d$, 
similar to the bubble sort. 
Thus, we only need to prove that $\psi^{(i,j)}$ is $l_\infty$ $1$-Lipschitz continuous.

For readability, we utilize $\psi(x)$ to denote $\psi^{(i,j)}(x)$ in the following proof.
Given $\forall x_1, x_2$ and $\psi^{(i,j)}(x)$, we are required to prove that:
\begin{equation}
\label{eq_pf:psi_0}
\begin{aligned}
    \| \psi(x_1) - \psi(x_2) \|_\infty = \max_k |\psi_k(x_1) - \psi_k(x_2)| \leq \max_k |x_{1,k} - x_{2,k}|
\end{aligned}
\end{equation}
If $x_{1,i} > x_{1,j}$ and $x_{2,i} > x_{2,j}$, Eq.~\eqref{eq_pf:psi_0} holds with $\psi(x_1)=x_1$, $\psi(x_2)=x_2$.
Similarly, if $x_{1,i} < x_{1,j}$ and $x_{2,i} < x_{2,j}$, Eq.~\eqref{eq_pf:psi_0} holds with $\max_k |\psi_k(x_1) - \psi_k(x_2)| = \max_k |x_{1,k} - x_{2,k}|$.
Therefore, we only need to consider the case that only one of $x_1$ and $x_2$ is modified by $\psi$.
Without loss of generality, let $x_{1,i}\leq x_{i,j}$ and $x_{2,i}\geq x_{2,j}$.
Firstly, we aim to prove that:
\begin{equation}
\label{eq_pf:psi_1}
\max_{k\in\{i,j\}} |\psi_k(x_1) - \psi_k(x_2)|
 = \max\left\{ \left| x_{1,j} - x_{2,i} \right|, \left| x_{1,i} - x_{2,j} \right| \right\}
\leq \max_{k\in\{i,j\}} |x_{1,k} - x_{2,k}|.
\end{equation}
\textbf{Case 1:} $x_{1,i} \leq x_{1,j} \leq x_{2,j} \leq x_{2,i}$ or $x_{1,i} \leq x_{2,j} \leq x_{1,j} \leq x_{2,i}$:

We have $\left| x_{1,j} - x_{2,i} \right| = x_{2,i}-x_{1,j} \leq x_{2,i} - x_{1,i}$ and $\left| x_{1,i} - x_{2,j} \right| = x_{2,j}-x_{1,i} \leq x_{2,i} - x_{1,i}$.
Thus, Eq.~\eqref{eq_pf:psi_1} holds in this case.

\textbf{Case 2:} $x_{1,i} \leq x_{2,j} \leq x_{2,i} \leq x_{1,j}$:

We have $\left| x_{1,j} - x_{2,i} \right| = x_{1,j} - x_{2,i} \leq x_{2,i} - x_{2,j}$.
Besides, $\left| x_{1,i} - x_{2,j} \right| = x_{2,j}-x_{1,i} \leq x_{2,i} - x_{1,i}$.
Thus, Eq.~\eqref{eq_pf:psi_1} holds in this case.

\textbf{Case 3:} $x_{2,j} \leq x_{2,i} \leq x_{1,i} \leq x_{1,j}$ or $x_{2,j} \leq x_{1,i} \leq x_{2,i} \leq x_{1,j}$:

We have $\left| x_{1,j} - x_{2,i} \right| = x_{1,j} - x_{2,i} \leq x_{1,j} - x_{2,j}$ and $\left| x_{1,i} - x_{2,j} \right| = x_{1,i} - x_{2,j} \leq x_{1,j} - x_{2,j}$.
Thus, Eq.~\eqref{eq_pf:psi_1} holds in this case.

\textbf{Case 4:} $x_{2,j} \leq x_{1,i} \leq x_{1,j} \leq x_{2,i}$:

We have $\left| x_{1,j} - x_{2,i} \right| = x_{2,i} - x_{1,j} \leq x_{2,i} - x_{1,i}$.
Besides, $\left| x_{1,i} - x_{2,j} \right| = x_{1,i} - x_{2,j} \leq x_{1,j} - x_{2,j}$.
Thus, Eq.~\eqref{eq_pf:psi_1} holds in this case.

Above all, Eq.~\eqref{eq_pf:psi_1} holds for $\psi(x)$ given $\forall x_1, x_2$.
Besides, the $\psi(x)$ does not modify the values of $\psi_k(x), k\notin\{i,j\}$, thus we have:
\begin{equation}
\label{eq_pf:psi_2}
    \max_{k\notin\{i,j\}} |\psi_k(x_1) - \psi_k(x_2)| 
    \leq \max_{k\notin\{i,j\}} |x_{1,k} - x_{2,k}| \leq \max_{k} |x_{1,k} - x_{2,k}|.
\end{equation}
Based on Eq.~\eqref{eq_pf:psi_1} and Eq.~\eqref{eq_pf:psi_2}, we can draw the conclusion illustrated in Eq.~\eqref{eq_pf:psi_0}.
Therefore, the function $g^\pi$ is $l_\infty$ $1$-Lipschitz continuous.
\end{enumerate}

\end{proof}

\subsection{Proof of Theorem~\ref{thm:policy_robust_radius}}
\label{sec:pf_thm_policy_robust_radius}
\begin{theorem*}
Given a SortRL policy $ \pi(a|s)=\mathds{1}\left( a = \arg\max_{a_i} g^\pi_i(s) \right)$, where $g^\pi_i(s)$ is an $l_\infty$ $1$-Lipschitz continuous function, 
the lower bound of the robust radius for $\pi$ can be expressed as follows:
\begin{equation*}
    \mathcal{R}(\pi, s)\geq \frac{1}{2}\operatorname{margin}(g^\pi, s) , \, \forall s\in \mathcal{S},
\end{equation*}
where $\operatorname{margin}(g^\pi, s)$ denotes the difference between the largest and second-largest action scores output by $g^\pi$ at state $s$.
\end{theorem*}

\begin{proof}
The proof of this theorem is simple and similar to that in the classification tasks, such as~\cite[Appendix B.1]{zhang2022rethinking} and~\cite[Appendix P]{li2019preventing}.
Given $\forall s\in\mathcal{S}$, we define $\hat{s}$ that satisfies $\| \hat{s}-s \|_\infty \leq \frac{1}{2}\operatorname{margin}(g^\pi,s)$.
For readability, we denote $a^*$ as the optimal action at state $s$, i.e. $a^* \coloneqq \arg\max_{a_i} g^\pi_i(s)$.
Then we have:
\begin{equation}
\label{eq_pf:score_greater_than}
\begin{aligned}
    g^\pi_{a^*}(\hat{s}) &\geq g^\pi_{a^*}(s) - \| \hat{s} - s \|_\infty \\
    & \geq g^\pi_{a^*}(s) - \frac{1}{2}\operatorname{margin}(g^\pi,s).\\
\end{aligned}
\end{equation}
Similarly, we have:
\begin{equation}
\label{eq_pf:score_less_than}
\begin{aligned}
    \max_{j\neq a^*} g^\pi_{j}(\hat{s}) &\leq \max_{j\neq a^*}\left\{ g^\pi_{j}(s) + \| \hat{s} - s \|_\infty \right\} \\
    & \leq \max_{j\neq a^*} g^\pi_{j}(s) + \frac{1}{2}\operatorname{margin}(g^\pi,s).\\
\end{aligned}
\end{equation}
Based on the above Eq.~\eqref{eq_pf:score_greater_than} and Eq.~\eqref{eq_pf:score_less_than}, we can obtain that:
\begin{equation}
\label{eq_pf:score_a_bigger_than_others}
\begin{aligned}
    g^\pi_{a^*}(\hat{s}) - \max_{j\neq a^*} g^\pi_{j}(\hat{s}) & \geq
    \left\{ g^\pi_{a^*}(s) - \frac{1}{2}\operatorname{margin}(g^\pi,s) \right\} - 
    \left\{ \max_{j\neq a^*} g^\pi_{j}(s) + \frac{1}{2}\operatorname{margin}(g^\pi,s) \right\}  \\
    &= g^\pi_{a^*}(s) - \max_{j\neq a^*} g^\pi_{j}(s) - \operatorname{margin}(g^\pi,s) \\
    &= 0. \\
\end{aligned}
\end{equation}
The Eq.~\eqref{eq_pf:score_a_bigger_than_others} indicates that: $\forall s\in\mathcal{S}, \| \hat{s} - s \|_\infty \leq \frac{1}{2}\operatorname{margin}(g^\pi, s)$, we have $\arg\max_{a_i} g^\pi_i(s) = \arg\max_{a_i} g^\pi_i(\hat{s})$.
In other words, the policy $\pi$ makes same decisions at state $s$ and $\hat{s}$.
Therefore, we have $\mathcal{R}(\pi, s)\geq \frac{1}{2}\operatorname{margin}(g^\pi, s) , \, \forall s\in \mathcal{S}$.
\end{proof}

\subsection{Proof of the Policy Bias}
\label{sec:pf_policy_bias}
\begin{statement}
Given a SortNet layer initialized with standard Gaussian distribution, input $\boldsymbol{x}^{(0)} = \boldsymbol{0}$, the output is greater than $0$:
\begin{equation*}
    \mathbb{E}_{\boldsymbol{b}^{(1, k)}\sim \mathcal{N}(\boldsymbol{0}, \boldsymbol{I})} [x^{(1)}_k] - x^{(0)}_k > 0, \, \forall 1\leq k \leq d_1, \, \rho \in [0,1),
\end{equation*}    
where $\boldsymbol{x}^{(l)}_{k}$ denotes the $k$-th unit in the $l$-th layer.
\end{statement}

\begin{proof}
In this proof, we utilize $\boldsymbol{b}$ and $\boldsymbol{w}$ to denote $\boldsymbol{b}^{(1, k)}$ and $\boldsymbol{w}^{(1, k)}$ correspondly for readability.
We can obtain that
\begin{equation*}
\begin{aligned}
    \mathbb{E}_{\boldsymbol{b}\sim \mathcal{N}(\boldsymbol{0}, \boldsymbol{I})} [x^{(1)}_k] - x^{(0)}_k 
    &= \mathbb{E}_{\boldsymbol{b}\sim \mathcal{N}(\boldsymbol{0}, \boldsymbol{I})} \left[ \boldsymbol{\omega}^{\mathrm{T}} \operatorname{sort}\left(\left|\boldsymbol{b}\right|\right) \right]\\
    &=  \mathbb{E}_{\boldsymbol{b}\sim \mathcal{N}(\boldsymbol{0}, \boldsymbol{I})} \left[ \sum_{i=1}^{d} (1-\rho)\rho^{i-1} |\boldsymbol{b}|_{[i]} \right]\\
    &\geq (1-\rho) \, \mathbb{E}_{\boldsymbol{b}\sim \mathcal{N}(\boldsymbol{0}, \boldsymbol{I})} \Big[ \max_i |b_i| \Big]\\
    &> 0 \\
\end{aligned}
\end{equation*}
Thus, the output of the network is biased.
\end{proof}

\section{Implementation Details of SortRL}

\subsection{Training Framework}
The pseudocode of the training framework is shown in Algo.~\ref{alg:pd_sortrl}.

\begin{algorithm}[htbp]
\caption{Training framework for \emph{SortRL}}
\label{alg:pd_sortrl}
\textbf{Input}: MDP $\mathcal{M}$ for the given task\\
\textbf{Parameter}: Hyperparameters $\theta\in\mathbb{R}^+$, $\lambda\in\mathbb{R}^+$. Number of iterations $N_{iter}$\\
\textbf{Output}: SortRL policy $\pi_S$.
\begin{algorithmic}[1] 
\STATE Initialize a typical DNN-based policy $\pi_T$ and a sortRL policy $\pi_S$ for the given task.
\STATE Train  $\pi_T$ in $\mathcal{M}$ with typical DRL algorithms, i.e. $\pi_T\leftarrow \arg\max _{\pi} \mathbb{E}_{s\sim\rho}[V_{\pi}(s)]$.
\STATE Construct an expert dataset $\mathcal{D}=\{s,a^*\}$ through performing $\pi_T$ in $\mathcal{M}$.
\FOR{$i \gets 1$ to $N_{iter}$}
\FOR{each mini-batch of $\mathcal{D}$}
\STATE Calculate the loss: $\mathcal{L}_{\pi_S} = \, \lambda \, \mathbb{E}_{(s,a^*)\sim \mathcal{D}}\big[\mathcal{L}_{\operatorname{CE}}\left(g^\pi(s), a^*\right)\big]
+ \mathbb{E}_{(s,a^*)\sim \mathcal{D}}\big[ \mathcal{L}_{\operatorname{Rob}}\left(g^\pi(s), \theta, a^*\right)\big]$ \\
\STATE Update $\pi_S$ with $\mathcal{L}_{\pi_S}$
\ENDFOR
\ENDFOR
\STATE \textbf{return} SortRL policy $\pi_S$.
\end{algorithmic}
\end{algorithm}

\subsection{Details of the Policy Network}
\subsubsection{Efficient Implementation}
Recall that the $l$-layer of the SortRL Policy is formulated as follows: 
\begin{equation*}
\begin{aligned}
    x_k^{(l)}=\left(\boldsymbol{w}^{(l, k)}\right)^{\mathrm{T}} \operatorname{sort}\left(\left|\boldsymbol{x}^{(l-1)}+\boldsymbol{b}^{(l, k)}\right|\right), \;
     \omega^{(l,k)}_i = (1-\rho)\rho^{i-1}, \, 1\leq l\leq M, \, 1\leq k\leq d_l, \\
\end{aligned}
\end{equation*}
where $d_l$ is the size of $l$-th network layer, $\rho \in [0,1)$ is a hyper-parameter.
$\operatorname{sort}(\boldsymbol{x}) \coloneqq \left[x_{[1]}, \cdots, x_{[d]}\right]^{\mathrm{T}}$ denotes sort operations, where $x_{[k]}$ is the $k$-th largest element of $\boldsymbol{x}\in \mathbb{R}^d$.
However, the SortRL policy is computationally inefficient for the high computational cost of $\operatorname{sort}$ operations.
To address this issue, the computation of $\operatorname{sort}$ operation needs to be paralleled. 
We utilize one possible scheme proposed in \cite{zhang2022rethinking}, which is described briefly as follows.
\begin{proposition*}
    For any $\rho \in [0,1)$, $\boldsymbol{w} \in \mathbb{R}^d$ is a vector where $w_i=(1-\rho) \rho^{i-1}, 1\leq i \leq d$.  For $\forall \boldsymbol{x} \in \mathbb{R}_{+}^d$, we have: 
    \begin{equation}
    \label{Eq:estimation}
        \boldsymbol{w}^{\mathrm{T}} \operatorname{sort}(\boldsymbol{x})=\mathbb{E}_{\boldsymbol{s} \sim \operatorname{Ber}(1-\rho)}\left[\max _i s_i x_i\right]
         = \mathbb{E}_{\boldsymbol{s} \sim \operatorname{Ber}(1-\rho)} \left[\lim _{p \rightarrow+\infty}\Big(\sum_i (s_i x_i)^p\Big)^{1 / p}\right]
    \end{equation}
    where $\boldsymbol{s}$ is a random variable following independent Bernoulli distribution with probability $1 - \rho$ being 1 and $\rho$ being 0.
\end{proposition*}
According to this proposition, we obtain an unbiased estimation of $\boldsymbol{w}^{\mathrm{T}} \operatorname{sort}(\boldsymbol{x})$ based on sampling on the given Bernoulli distribution, which can be computed parallelly and efficiently.
Besides, Zhang et al.~\cite{zhang2022rethinking} points out that $\max _i s_i x_i$  in Eq.~\eqref{Eq:estimation} makes Gradient-Decent based optimization of the network challenging.
To alleviate this issue, the hyper-parameter $p$ is utilized to give a smooth approximation of the maximum operation.
In this work, during the training process, $p$ gradually increases from a small value $8$ until reaching a large value $10^3$.

\subsubsection{Normalization}

As described in Sec.~\ref{sec:training_framework}, the output of each layer in $g^\pi$ is biased (always being non-negative) under random initialization, and thus cannot be trained with typical DRL algorithms directly.
As illustrated in Sec.~\ref{sec:pf_policy_bias}, we give a brief proof of the bias of the policy initialized with standard Gaussian Distribution with zero input.
The bias of each layer is then fed into subsequent layers and thus can be accumulated through forward propagation.
This makes the outputs in upper layers linearly increase, leading to unstable or ineffective outputs of the network.

In order to address this issue, inspired by $l_\infty$-dist net~\cite{zhang2021towards}, normalization is utilized after each intermediate layer to control the scale of the output, i.e.
$\boldsymbol{x}^{(l)} \leftarrow \boldsymbol{x}^{(l)} - \mathbb{E}\left[\boldsymbol{x}^{(l)}\right]$, where $\mathbb{E}\left[\boldsymbol{x}^{(l)}\right]$ is the mean value of the $l$-th layer output.
The estimation of $\mathbb{E}\left[\boldsymbol{x}^{(l)}\right]$ is obtained based on the mini-batch during training, while the moving average value is utilized during evaluation.
Note that this operation is quite similar to Batch Normalization~\cite{ioffe2015batch} without scaling operation because the scaling operation may change the Lipschitz constant.

\subsubsection{Scaled cross-entropy loss}
As shown in Eq.~\eqref{eq:pi_loss}, the student policy $\pi_S$ is trained to mimic the teacher policy with the Cross-Entropy loss, i.e. $\mathcal{L}_{\operatorname{CE}}\left(g^\pi(s), a^*\right)$.
The  Cross-Entropy loss is invariant to the shift operation but not scaling, such as multiplying a constant to the network outputs.
However, the SortRL policy is $1$-Lipschitz continuous, thus cannot adjust the scale of the output to match the Cross-Entropy loss.
Therefore, we utilize $\mathcal{L}_{\operatorname{CE}}\left(\mu\cdot g^\pi(s), a^*\right)$, where $\mu$ is a learnable scalar~\cite{zhang2021towards}.
Note that $\mu$ has no influence on the agent's decision.
In this work, $\mu$ is initialized as $1.0$ and adjusted by the optimizer automatically during training.

\section{Experiment Details of Classic Control}

\subsection{Environment Settings}

\subsubsection{Environments and Baseline Methods}
We utilize the gym environments and PPO experts proposed in Stable Baselines 3~\cite{stable-baselines3}, which is a reliable implementation of mainstream DRL algorithms.
In order to  evaluate the policy vulnerability to perturbations at each observation dimension, the observation normalization provided by PPO experts is utilized to  eliminate the dimensional influence.

\subsubsection{Adversary Settings}
In this experiment, we utilize the Projected Gradient Decent (PGD) attack~\cite{madry2018towards} as the adversary $\nu$, which is formulated as follows.
For a given state $s$, the untargeted PGD attacks attempt to change the policy decision with $K$ iterations updates:
\begin{equation}
\label{eq_detail:control_0}
    s^{k+1} = \operatorname{CLIP}_{s,\epsilon}\left( s^k + \eta\cdot \operatorname{sign}\left( \nabla_{s^k} \mathcal{L}_{\operatorname{CE}}(\pi(s^k), a^*) \right) \right), 
    \quad s^0 = s, \, 0<k<K,
\end{equation}
where $\mathcal{L}_{\operatorname{CE}}$ denotes the cross-entropy loss, $\eta$ denotes step size, $K$ denotes step counts, and $a^*=\arg\max_a \pi(s^k)$ denotes the optimal action given by the agent.
$\operatorname{CLIP}_{s,\epsilon}$ is utilized to guarantee that $\|s^k-s\|_{\infty} \leq \epsilon$.
Note that $g^\pi$ is utilized instead of $\pi$ in Eq.~\eqref{eq_detail:control_0} because $\pi(s)$ is not differentiable.
In this experiment, the step size $\eta = \epsilon/10$, and $K=10$.

\subsection{Training Details}
All the SortRL policies are trained with AdamW~\cite{loshchilov2018decoupled} optimizer with learning rate $\alpha=0.02$, weight decay $0.02$, and batch size $512$.
PPO experts trained by Stable Baselines 3 are utilized as teacher policies.
The teacher dataset is composed of $50K$ states and corresponding expert actions.
Each student policy network is composed of $5$ layers with width=$640$.
$\rho=0.3$.
During training, each $\pi_S$ is trained with $2K$ iterations, while most policies converge at $1K$ iterations.
The $\theta$ is set $2\epsilon$, where $\epsilon$ is the max strength of perturbations to resist.
In this work, all SortRL policies are implemented based on Pytorch.
The policies are trained on an NVIDIA RTX 3090 GPU, Ubuntu 22.04, and the CUDA version is 11.8

\section{Experiment Details of Video Games}

\subsection{Environment Settings}
\subsubsection{Environments}
In this experiment, all the methods are evaluated in the same environments with the same gym wrappers.
Two categories of games are chosen, including four Atari games and two ProcGen games.
At each step of Atari games, the agent obtains an image observation composed of $84\times 84\times 1$ grey-scaled pixels with no frame-stacking.
All rewards are clipped between $[-1,1]$.
Each pixel is scaled to be between $0\sim1$. 

The observation in ProcGen is an image with $64\times 64\times 3$ RGB pixels.
Each pixel value is scaled to be between $0\sim1$. 
Each time, the RL policy is trained on a finite set of levels (train mode)
and tested on the full distribution of levels (test mode), which is designed to test the generalization ability of the policies.

\subsubsection{Baseline Methods}
In this section, we mainly introduce the Bootstrapped Opportunistic Adversarial Curriculum Learning (BCL) method, which is a novel flexible adversarial curriculum learning framework and enhances the robustness of existing robust RL methods.
BCL can be combined with various methods.
In this work, \emph{BCL-RADIAL} means improving the  robustness of RADIAL with BCL framework under strong adversaries.
\emph{BCL-RADIAL+AT} means run BCL-RADIAL until it reaches a point in the curriculum at which its performance degrades significantly, then switch to Adversarial Training (AT) for the remainder of the curriculum.
\emph{BCL-MOS} means the combination of BCL with Maximum Opportunistic Skipping (MOS).
In detail, we always choose to skip to the smallest against which the current model is not (yet) robust, the most opportunistic version of the algorithm is obtained.

\subsubsection{Adversary Settings}
In Sec. 4.2, we utilize 10 steps PGD attack with step size $\epsilon/10$ on Atari and ProcGen, which is the same as the experiments on the classic control.
In Sec. 4.3, i.e. experiments on video games with stronger adversaries, we utilize the same adversary setting as that in the BCL work~\cite{wu2022robust}.
In detail, we utilize 30 steps PGD, FGSM, RI-FGSM-Multi, and RI-FGSM-Multi-T on Atari tasks, while only utilizing 30 steps PGD on ProcGen.
The step size of the FGSM-based adversary is $0.375$.

\subsection{Action Certification Rate}
In this work, we utilize Action Certification Rate (ACR)~\cite{zhang2020robust} to evaluate the certified robustness of our method.
ACR is defined as the proportion of the actions during rollout that are guaranteed unchanged with any adversary $\nu\in\mathcal{B}_{\epsilon}^{\infty}$, which is calculated as follows.

Given the policy $\pi$ for SA-MDP $\widetilde{\mathcal{M}}$, we perform $\pi$ in the corresponding typical MDP $\mathcal{M}$ and collect $N$ states $\{s_i\}$.
As described in Theorem 2, $\forall s\in\mathcal{S}$, if $\operatorname{margin}(g^\pi, s) \geq 2\epsilon$, we can obtain that $\pi(s)=\pi(\hat{s}), \, \forall \hat{s}\sim \nu(s)$, i.e. the SortRL $\pi$ can resist attacks from any adversary $\nu\in \mathcal{B}_{\epsilon}^{\infty}(s)$.
Thus, we can obtain that:
\begin{equation*}
    ACR = \frac{\left| \left\{s_i|\operatorname{margin}(g^\pi, s_i)\geq 2\epsilon\right\} \right|}{| \{s_i\} |}
\end{equation*}

\subsection{Additional Experiment Results}

The experiment results on video games with stronger adversaries are given in Table
~\ref{tab:atari_res_big_eps} (Atari) and Table~\ref{tab:procgen_res_big_eps} (ProcGen).
As shown in the tables, our method achieves state-of-the-art performance compared to existing methods, including BCL-based methods.
Especially in tasks with $\epsilon>15/255$, SortRL outperforms baseline methods with considerable advantages.
Take \emph{RoadRunner}  with $\epsilon={20}/{255}$ as an instance, SortRL achieves an episode reward of $32433$  and outperforms existing state-of-the-art BCL-RADIAL-AT ($25325$) by approximately $28\%$.
Besides, we can also observe significant improvement of BCL-based methods, especially in tasks with $5/255\leq\epsilon\leq 10/255$, which demonstrates the effectiveness of curriculum learning in robust RL.
In the future, the combination of SortRL and BCL can be further studied.

\begin{table*}[htbp]
\centering
\begin{tabular}{c|l|lll}
\hline
\multicolumn{1}{l|}{\textbf{Task}} & {\textbf{Model/Metric}} & \multicolumn{3}{c}{\textbf{Episode Reward}} \\ \hline
\multicolumn{1}{l|}{} & {$\epsilon$} & $10/255$ & $15/255$ & $20/255$ \\ \hline
\multirow{6}{*}{Freeway} 
& {DQN} & $0.0\pm0.0$ & $0.0\pm0.0$ & $0.0\pm0.0$ \\
& {SA-DQN} & $19.3\pm0.4$ & $19.3\pm0.3$ & $20.0\pm0.3$ \\
& {RADIAL-DQN} & $17.1\pm0.3$ & $13.4\pm0.2$ & $7.9\pm0.3$ \\
& {BCL-MOS-AT} & $31.1\pm0.3$ & \underline{$25.9\pm0.4$} & $20.8\pm0.3$ \\
& {BCL-RADIAL} & \bm{$33.4\pm0.1$} & $25.9\pm0.6$ & \underline{$21.2\pm0.5$} \\
 \cline{2-5}  
 \rowcolor{Gray}\cellcolor{white} &{SortRL-DQN} & \underline{$33.1\pm0.4$} & \bm{$30.3\pm0.7$} & \bm{$27.2\pm0.8$} \\ \hline
 \hline
 \multicolumn{1}{l|}{} & {$\epsilon$} & $5/255$ & $10/255$ & $15/255$ \\ \hline
\multirow{7}{*}{\begin{tabular}[c]{@{}c@{}}Bank\\ Heist\end{tabular}} 
& {DQN} & $0.0\pm0.0$ & $0.0\pm0.0$ & $0.0\pm0.0$ \\
& {SA-DQN} & $1126.0\pm32.0$ & $63.0\pm3.5$ & $16.0\pm1.6$ \\
& {RADIAL-DQN} & $518.5\pm16.7$ & $0.0\pm0.0$ & $0.0\pm0.0$ \\
& {BCL-MOS-AT} & $1095.5\pm6.2$ & $664.0\pm60.6$ & $586.5\pm105.6$ \\
& {BCL-RADIAL} & \underline{$1225.5\pm4.9$} & \underline{$1223.5\pm4.1$} & $228.5\pm13.9$ \\
& {BCL-RADIAL+AT} & $1093.0\pm5.3$ & $1010.5\pm8.0$ & \underline{$961.5\pm9.2$} \\
 \cline{2-5}
\rowcolor{Gray}\cellcolor{white} & {SortRL-DQN} & \bm{$1299.1\pm7.2$} & \bm{$1265.5\pm9.8$} & \bm{$1193.8\pm15.2$} \\
 \hline
 \hline
 \multicolumn{1}{l|}{} & {$\epsilon$} & $5/255$ & $10/255$ & $15/255$ \\ \hline
\multirow{7}{*}{\begin{tabular}[c]{@{}c@{}}Road\\ Runner\end{tabular}} & {DQN} & $0.0\pm0.0$ & $0.0\pm0.0$ & $0.0\pm0.0$ \\
& {SA-DQN} & $985\pm207$ & $0.0\pm0.0$ & $0.0\pm0.0$ \\
& {RADIAL-DQN} & $7195\pm929$ & $495\pm116$ & $0.0\pm0.0$ \\
& {BCL-MOS-AT} & \underline{$40060\pm1828$} & $15785\pm1124$ & $1195\pm180$ \\
& {BCL-RADIAL} & $37865\pm1082$ & \bm{$37865\pm1082$} & $6350\pm590$ \\
& {BCL-RADIAL+AT} & \bm{$42490\pm1309$} & \underline{$37665\pm1563$} & \underline{$25325\pm1057$} \\
 \cline{2-5}
 \rowcolor{Gray}\cellcolor{white} &{SortRL-DQN}  & $39924\pm1429$ & {$35297\pm1566$} & \bm{$32433\pm1719$}  \\ \hline

\end{tabular}
\caption{The experiment results on the Atari video games.
The best results are \textbf{boldfaced}, while the second best ones are \underline{underlined}.
The \colorbox{Gray}{{gray row}} denotes the most robust method, selected based on the score $\sum_\epsilon R_{\epsilon}$, where $R_{\epsilon}$ is the mean episode reward given perturbation strength $\epsilon$.
}
\label{tab:atari_res_big_eps}
\end{table*}

\begin{table*}[htbp]
\centering
{%
\begin{tabular}{c|l|l|lll}
\hline
\textbf{Task} & \multicolumn{2}{c|}{\textbf{Model/Metric}} & \multicolumn{3}{c}{\textbf{Episode Reward}} \\ \hline
\multicolumn{1}{l|}{} & \multicolumn{1}{l|}{$\epsilon$} & Env. Type & $10/255$ & $20/255$ & $40/255$ \\ \hline
 \multirow{10}{*}{Jumper} 
 & \multicolumn{1}{l|}{\multirow{2}{*}{PPO}} & Train & $3.42\pm0.15$ & $3.61\pm0.15$ & $2.94\pm0.14$ \\
 & \multicolumn{1}{l|}{} & Eval & $2.81\pm0.14$ & $2.62\pm0.14$ & $2.50\pm0.14$ \\ \cline{2-6} 
 & \multicolumn{1}{l|}{\multirow{2}{*}{RADIAL-PPO}} & Train & $5.43\pm0.16$ & $2.45\pm0.14$ & $1.44\pm0.11$ \\
 & \multicolumn{1}{l|}{} & Eval & $3.03\pm0.14$ & $2.04\pm0.13$ & $1.44\pm0.11$ \\ \cline{2-6} 
 & \multicolumn{1}{l|}{\multirow{2}{*}{BCL-MOS(V)-AT}} & Train & $8.15\pm0.12$ & $8.40\pm0.12$ & $7.84\pm0.13$ \\
 & \multicolumn{1}{l|}{} & Eval & $4.64\pm0.16$ & $4.65\pm0.16$ & $4.41\pm0.16$ \\ \cline{2-6} 
 & \multicolumn{1}{l|}{\multirow{2}{*}{BCL-MOS(R)-AT}} & Train & $8.29\pm0.12$ & $8.40\pm0.12$ & $6.93\pm0.15$ \\
 & \multicolumn{1}{l|}{} & Eval & $4.29\pm0.16$ & $4.09\pm0.16$ & $3.85\pm0.15$ \\ \cline{2-6} 
 \rowcolor{Gray}\cellcolor{white} & & Train & \bm{$9.10\pm0.29$} & \bm{$9.10\pm0.29$} & \bm{$9.10\pm0.30$} \\ 
 \rowcolor{Gray}\cellcolor{white} &  {\multirow{-2}{*}{SortRL-PPO (Ours)}} & Eval & \bm{$4.67\pm0.39$} & \bm{$4.70\pm0.39$} & \bm{$4.70\pm0.40$} \\ \hline
\end{tabular}%
}
\caption{The experiment results on the Jumper tasks.}
\label{tab:procgen_res_big_eps}
\end{table*}

\subsection{Training Details}
In this experiment, SortRL policies are trained with learning rate $\alpha=0.02$, weight decay $0.02$, and batch size $512$.
DQN experts are utilized as teacher policies.
The teacher dataset is composed of $100K$ states and corresponding expert actions.
According to ~\cite{zhang2022rethinking}, we construct the student policy network with $6$ layers with width=$1280$, followed by two linear layers with width $256$ to enhance the expressive power of the whole network.
Interval Bound Propagation (IBP) \cite{gowal2018effectiveness} is utilized to ensure the robustness of the linear layers.
$\rho=0.3$.
During training, each $\pi_S$ is trained with $2K$ iterations, while most policies converge at $1K$ iterations.
The $\theta$ is set $2.2\epsilon$, where $\epsilon$ is the max strength of perturbations to resist.
In this work, all SortRL policies are implemented based on Pytorch.
The policies are trained on an NVIDIA RTX 3090 GPU, Ubuntu 22.04, and the CUDA version is 11.8.

\end{document}